\documentclass{article} %
\usepackage{iclr2025_conference,times}

\usepackage{amsmath,amsfonts,bm}

\def\eqref#1{equation~\ref{#1}}

\def\1{\bm{1}}

\DeclareMathAlphabet{\mathsfit}{\encodingdefault}{\sfdefault}{m}{sl}
\SetMathAlphabet{\mathsfit}{bold}{\encodingdefault}{\sfdefault}{bx}{n}

\usepackage{hyperref}
\usepackage{url}
\usepackage{amsthm}
\usepackage{xcolor} %
\usepackage{graphicx}    %
\usepackage{subcaption}  %
\usepackage{booktabs}
\usepackage{multirow}
\usepackage[utf8]{inputenc}
\usepackage{tcolorbox}
\usepackage{wrapfig}
\usepackage{color}
\theoremstyle{plain}
\newtheorem{theorem}{Theorem}[section]

\newtheorem{lemma}[theorem]{Lemma}

\theoremstyle{definition}

\newtheorem{assumption}[theorem]{Assumption}
\theoremstyle{remark}

\newtcolorbox{promptbox}[1][]{
    colback=gray!10,        %
    colframe=gray!50,       %
    coltitle=black,         %
    boxrule=0.5mm,          %
    sharp corners,          %
    fonttitle=\bfseries,    %
    width=\textwidth,       %
    before skip=10pt,       %
    after skip=10pt,        %
    #1                       %
}

\newtcolorbox{takeawaybox}[1][]{
    colback=blue!10,        %
    colframe=blue!50,       %
    coltitle=black,         %
    boxrule=0.5mm,          %
    sharp corners,          %
    fonttitle=\bfseries,    %
    width=\textwidth,       %
    before skip=10pt,       %
    after skip=10pt,        %
    #1                       %
}

\title{Can In-context Learning Really\\ Generalize to Out-of-distribution Tasks?}

\author{%
  Qixun Wang\textsuperscript{1} \qquad Yifei Wang \textsuperscript{2} \qquad Yisen Wang \textsuperscript{1,3}\thanks{Corresponding author: Yisen Wang (yisen.wang@pku.edu.cn).} \qquad Xianghua Ying \textsuperscript{1} \thanks{Corresponding author: Xianghua Ying (xhying@pku.edu.cn).} \\
\textsuperscript{1} National Key Laboratory of General Artificial Intelligence,\\School of Intelligence Science and Technology, Peking University\\
\textsuperscript{2} CSAIL, Massachusetts Institute of Technology\\
\textsuperscript{3} Institute for Artificial Intelligence, Peking University\\
}

\begin{document}

\maketitle

\begin{abstract}
In this work, we explore the mechanism of in-context learning (ICL) on out-of-distribution (OOD) tasks that were not encountered during training. To achieve this, we conduct synthetic experiments where the objective is to learn OOD mathematical functions through ICL using a GPT-2 model. We reveal that Transformers may struggle to learn OOD task functions through ICL. Specifically, ICL performance resembles implementing a function within the pretraining hypothesis space and optimizing it with gradient descent based on the in-context examples. Additionally, we investigate ICL's well-documented ability to learn unseen abstract labels in context. We demonstrate that such ability only manifests in the scenarios without distributional shifts and, therefore, may not serve as evidence of new-task-learning ability. Furthermore, we assess ICL's performance on OOD tasks when the model is pretrained on multiple tasks. Both empirical and theoretical analyses demonstrate the existence of the \textbf{low-test-error preference} of ICL, where it tends to implement the pretraining function that yields low test error in the testing context. We validate this through numerical experiments. This new theoretical result, combined with our empirical findings, elucidates the mechanism of ICL in addressing OOD tasks.
\end{abstract}

\section{Introduction}
Pretrained large language models (LLMs) can perform in-context learning (ICL) \citep{brown2020language}, where providing a few examples of input-output pairs and a query example improves the model's ability to generate the desired output, compared to zero-shot predictions. Understanding ICL’s ability to learn out-of-distribution (OOD) input-output relationships, which are unseen during training, has recently gained significant attention.

Recent studies have demonstrated that ICL can tackle seemingly new tasks. For instance, \citet{garg2022can, raventos2024pretraining, zhang2023trained, akyurek2022learning} found that ICL can learn new linear regression weights after pretraining on a large set of weight vectors. Moreover, \citet{pan2023context, kossen2024context, vacareanu2024words} revealed that real-world LLMs like Llama-2 \citep{touvron2023llama} and GPT-4 \citep{achiam2023gpt} are capable of solving artificially constructed tasks likely unseen in their pretraining data, such as a classification task with abstract labels.

However, another line of research \citep{yadlowsky2023pretraining, ahuja2023closer} has raised a contrasting view, showing that ICL struggles to generalize to OOD tasks where there are distributional shifts in either the input distribution $P(X)$ or the input-label mapping $P(Y|X)$. These findings raise several important questions:
\begin{quote}
\emph{Can ICL really learn new input-output mappings from the context? What underlying mechanism of ICL determines its performance on OOD tasks?}
\end{quote}

In this work, we aim to consolidate previous findings by addressing these questions. First, we empirically show that when pretrained on a specific function class, the OOD performance of ICL approaches that of a model from the same function class optimized via gradient descent. This suggests that ICL tends to implement functions encountered during pretraining, which could explain its failure on OOD tasks that significantly deviate from the training distribution. Furthermore, we reproduce the widely observed phenomenon that ICL can perform classification with abstract labels. We find that solving such tasks requires retrieving similar labels from the context, a capability that can be acquired through pretraining on analogous tasks. This implies that success in such tasks of ICL may not indicate an inherent ability to learn new tasks. Finally, we explore scenarios in which the model is pretrained on multiple tasks, empirically uncovering the algorithm selection mechanism for OOD tasks. We also provide a comprehensive theoretical framework for understanding the algorithm-selection mechanism for ICL. Our contributions are summarized as follows:
\begin{enumerate}
    \item We empirically show that ICL tends to implement the pretraining function based on the downstream task context, highlighting its limitation in solving OOD tasks (Section \ref{subsec:main}).
    \item We further investigate ICL's ability to perform classification with unseen abstract labels. Although this appears to be evidence of ICL learning OOD tasks, we find that such tasks can be solved by retrieving similar examples from the context. This retrieval ability can arise from training on tasks with more diverse abstract labels (Section \ref{subsec:abstract-label-learning}) and only emerges when the testing function is in distribution (Section \ref{subsec:abstract-label-learning-iid}). Additionally, we find that pretrained Llama-3-8B \citep{dubey2024llama} and Llama-2-7B fails to learn OOD functions through ICL in a synthetic word classification task (Section \ref{sec:llama-ood-task}), further confirming ICL’s limitations in OOD scenarios.
    \item Finally, we explore the ICL's behavior when trained on multiple tasks, and observe that the algorithm selection mechanism broadly occurs in OOD scenarios. We theoretically prove the \textbf{low-test-error} preference of ICL prediction, i.e., the ICL prediction prefers to implement the pretraining function with lower test error (Section \ref{sec:theo-algo-selection}). We also validate our theory with numerical experiments (\ref{sec:exp-algo-selection}).
\end{enumerate}

\section{Existing Theoretical Predictions About ICL's OOD Behavior}
\label{sec:theoretical-predictions-of-ICL-OOD}
Previous literature have provided some theoretical insights into the behavior of ICL. Here we briefly review some of the representative findings. 1) \textbf{ICL makes Bayesian predictions.} \citep{xie2021explanation, wies2024learnability, zhang2023and} theoretically demonstrated that ICL behaves like a Bayesian-optimal predictor, i.e., it will infer a task concept based on the given test context, and then predict using the inferred task and the input prompt. However, these Bayesian frameworks don't depict the concrete process of how the task concept is inferred, especially for OOD scenarios. 2) \textbf{ICL implements gradient descents.} \cite{von2023transformers, zheng2024mesa} construct specific Transformer weights on which the ICL prediction is equivalent to a linear regression predictor optimized by gradient descent. 3) \textbf{ICL implements algorithm selection.}  \cite{bai2024transformers,wang2024context} demonstrate the existence of Transformers that can realize algorithm selection between linear classification and regression by constructing specific Transformer weights. 4) \textbf{ICL performs retrieval.} \cite{li2024linonlinear} proves that a trained nonlinear Transformer will concentrate its attention of the query on the in-context examples possessing similar features to that of the query.

These theoretical findings may appear disparate, as they describe different aspects of ICL under varying assumptions and settings. Furthermore, most of them remain largely unexplored empirically, particularly on large-scale nonlinear Transformers. In the following sections, we aim to provide a unifying perspective on ICL by conducting experiments with deep nonlinear Transformers on real-world OOD tasks.

\section{Exploring the ICL Performance on OOD Tasks}
\label{sec:icl-makes-id-predictions}
\subsection{GPT-2 Implements Functions Classes Seen During ICL Pretraining}
\label{subsec:main}
\textbf{Evaluating GPT-2 on unseen mathematical function classes.} To investigate the ICL performance on new tasks that are unseen during training, following \citet{garg2022can}, we train a GPT-2 \citep{radford2019language} from scratch on some simple functions and evaluate it on functions different from the training ones. Denote the Transformer model parameterized by $\theta$ as $M_\theta$. The pretraining objective is: 
\begin{equation}
    \min_\theta \frac{1}{T} \sum_{i=1}^{T}\mathbb{E}_{f\sim \mathcal{F}}[\|M_\theta(\mathcal{S}_{i}\oplus \boldsymbol{x}_{i+1})- f(\boldsymbol{x}_{i+1}) \|^2_2 ],
\end{equation}
where $\mathcal{S}_i=[\boldsymbol{x}_1\oplus y_1 \oplus \boldsymbol{x}_2 \oplus y_2 \oplus ... \oplus \boldsymbol{x}_i \oplus y_i]\in \mathbb{R}^{d \times 2i}$ is the context of length $i$, $\oplus$ denotes concatenation. $\boldsymbol{x}_i\in \mathbb{R}^d$ are sampled from a standard Gaussian distribution $\mathcal{N}(0,1)$ with dimension $d=20$. Let $y_i=f(\boldsymbol{x}_i)$ represent the labels, with $\mathcal{F}$ denoting the hypothesis class to which $f$ belongs. We train three separate GPT-2 models on three different function classes $\mathcal{F}$: linear regression (LR), quadratic regression (QR, element-wise square followed by linear regression), and a 2-layer ReLU network (ReLU NN, detailed descriptions are in Appendix \ref{app:exp-detail-1}). We then evaluate their ICL performance on these three tasks. Note that even when the testing and training functions are i.i.d. sampled from the same task, the specific instances of the testing functions remain unseen during training. For comparison, we also train models within the corresponding $\mathcal{F}$ with gradient descent (GD) using the testing in-context examples (details in Appendix \ref{app:exp-detail-1}).

\textbf{Observations.} We plot the test error on the three tasks in Figure \ref{fig:main} and observe that: 1) (an existing finding in \cite{garg2022can}) when evaluated on the same task $\mathcal{F}$ as pretraining, ICL can reach near-zero test error. 2) (our novel finding) when evaluated on a new task, ICL performs similarly to the corresponding model of the pretraining function class optimized by GD given enough in-context examples. This indicates that the ICL prediction implements the training function classes. 3) (our novel finding) The models trained on linear and quadratic regression exhibit a double descent error curve \citep{nakkiran2019more}, characterized by a high error when given exact $d$ examples and evaluated on a new task. This further demonstrates that ICL implements the ID predictions, as the double descent curve is a distinctive phenomenon unique to linear regression models. 

\begin{figure}[htbp]
\vspace{-10pt}
    \centering
    \begin{subfigure}{0.26\textwidth}
        \centering
        \includegraphics[width=\textwidth]{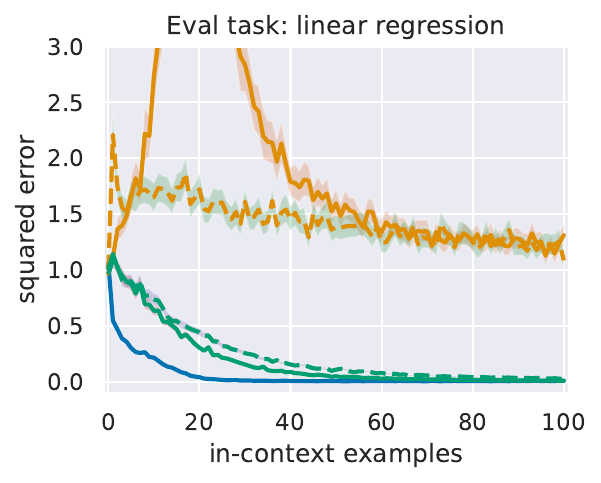}
        \vspace{-12pt}
        \label{fig:main-sub1}
        \vspace{-3pt}
        \caption{Evaluated on LR}
    \end{subfigure}
    \begin{subfigure}{0.26\textwidth}
        \centering
        \includegraphics[width=\textwidth]{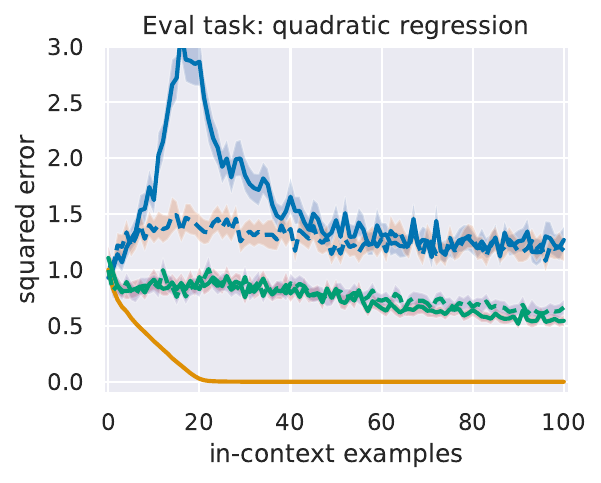}
        \vspace{-12pt}
        \label{fig:main-sub2}
        \vspace{-3pt}
        \caption{Evaluated on QR}
    \end{subfigure}
    \begin{subfigure}{0.44\textwidth}
        \centering
        \includegraphics[width=\textwidth]{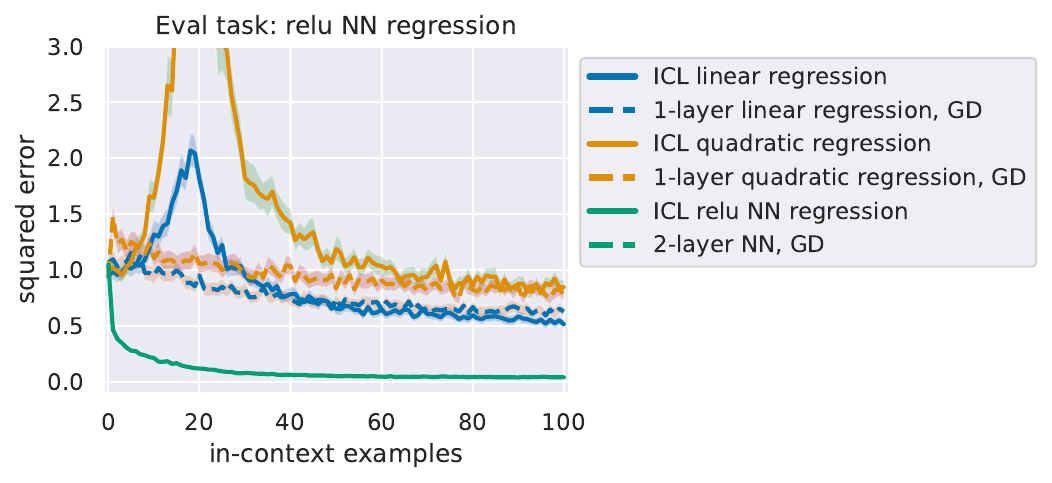}
        \vspace{-12pt}
        \label{fig:main-sub3}
        \vspace{-3pt}
        \caption{Evaluated ReLU NN}
    \end{subfigure}
    \caption{The ICL test error of Transformers trained on different function classes (solid lines) and the performance of the models from the corresponding pretraining functions classes trained by gradient descent (GD) using the in-context examples (dashed lines). Y-axis: test square error. X-axis: context length. In all evaluation tasks, we observe that as the test context length increases, the ICL performance of the Transformer pretrained on a particular function class closely approaches that of the model from this function class trained by GD.}
    \label{fig:main}
\end{figure}

\subsection{Real-world LLMs Tend to Make In-distribution Predictions during ICL}
\label{sec:realLLM-reverse}
In this section, we will demonstrate how the tendency of ICL to perform ID predictions manifests in real-world LLMs. To this end, we designed a task involving predicting labels with letters reversed. In some basic tasks like outputting antonyms or translating from English to French, all the letters of the original labels are reversed (e.g., ``positive"$\rightarrow$``evitisop"). We found that in this task, a pretrained Llama-3-8B \citep{dubey2024llama} tend to output the reversed result of the query word rather than first predicting the correct label and then reversing it. Although both reversal tasks are uncommon, directly outputting the reversed version of a word is relatively more common than first reasoning and then outputting the reversed prediction. Therefore, this result reflects to some extent that LLMs, when performing ICL, are more inclined to implement ID tasks. See Appendix \ref{app:realLLM-reverse} for more details.

\begin{figure}[htbp]
    \centering
    \begin{subfigure}{0.24\textwidth}
        \centering
        \includegraphics[width=\textwidth]{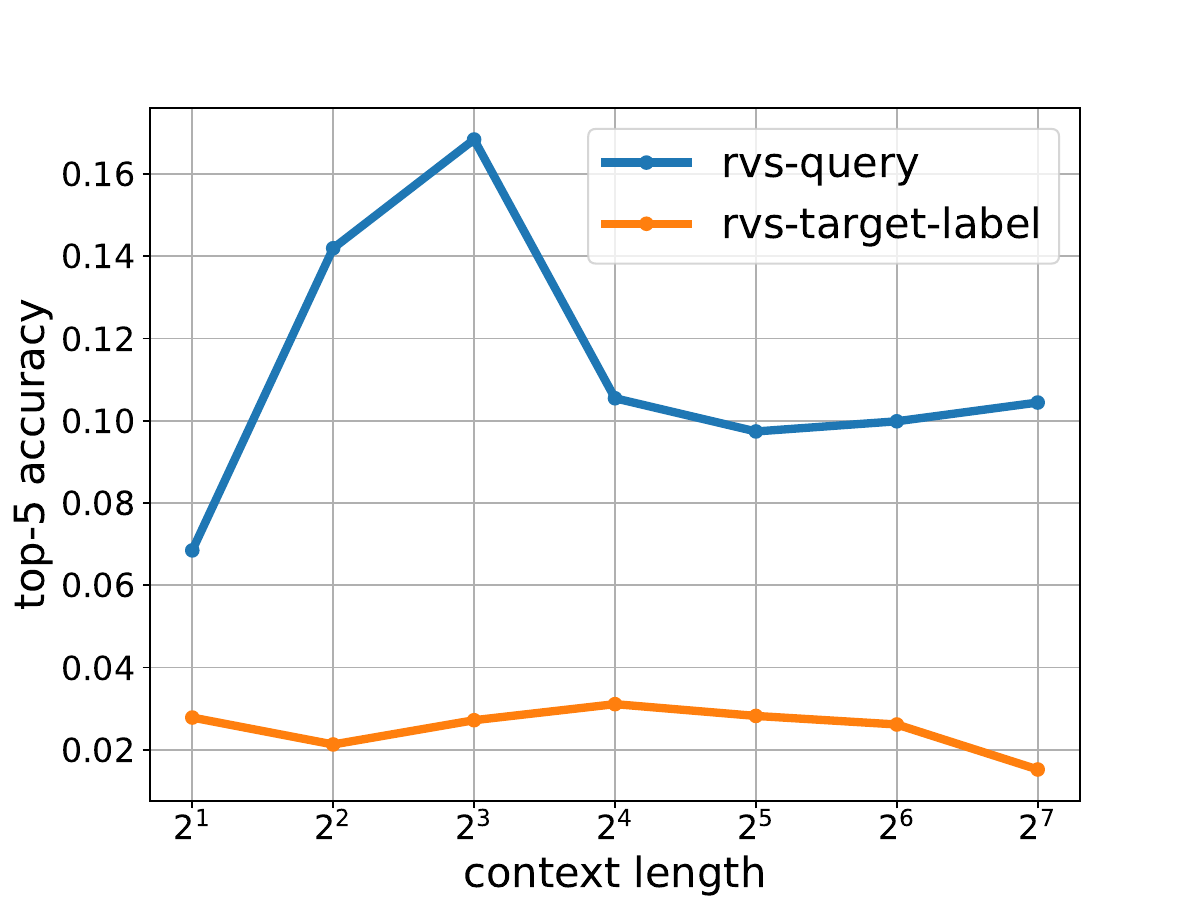}
        \vspace{-6pt}
        \label{fig:real-LLM-reverse1}
        \caption{Antonym}
    \end{subfigure}
    \begin{subfigure}{0.24\textwidth}
        \centering
        \includegraphics[width=\textwidth]{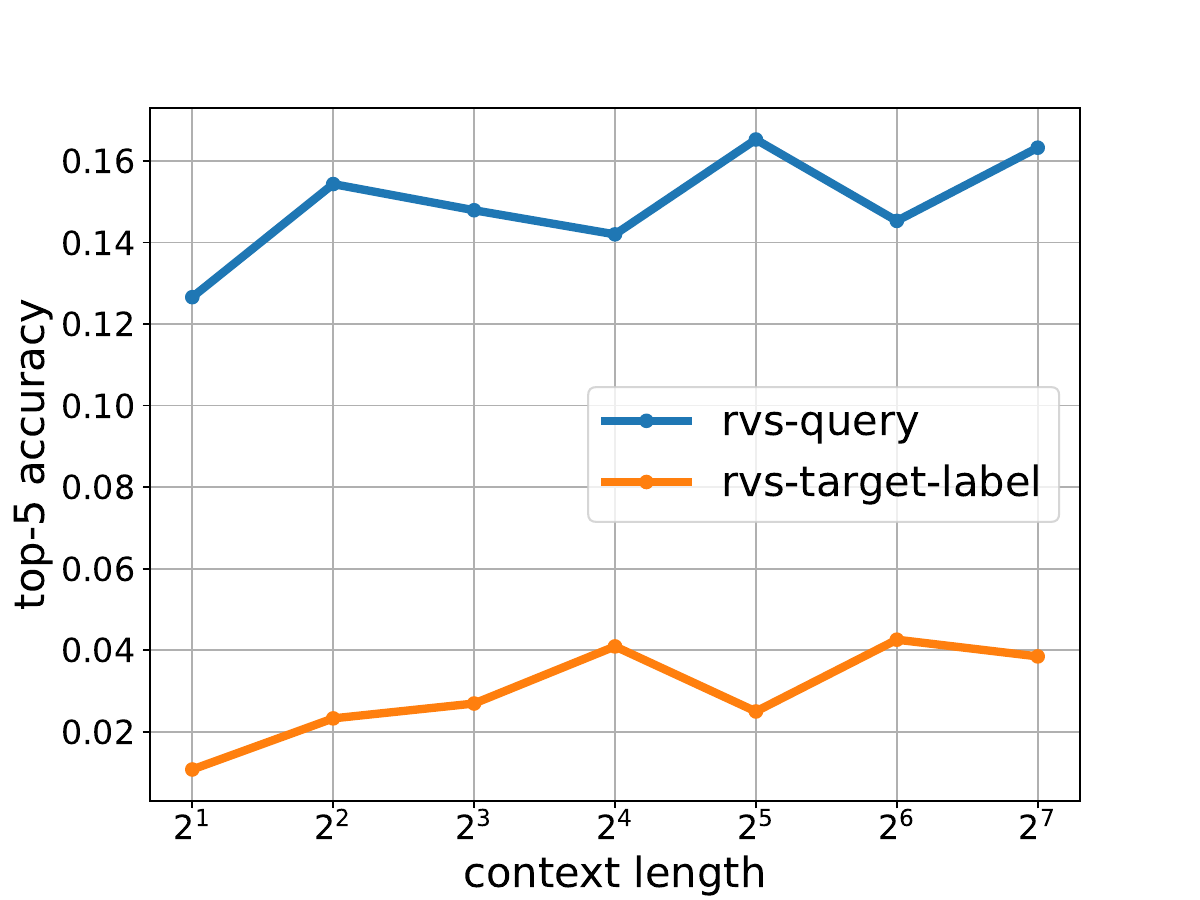}
        \vspace{-6pt}
        \label{fig:real-LLM-reverse2}
        \caption{Country-capital}
    \end{subfigure}
    \begin{subfigure}{0.24\textwidth}
        \centering
        \includegraphics[width=\textwidth]{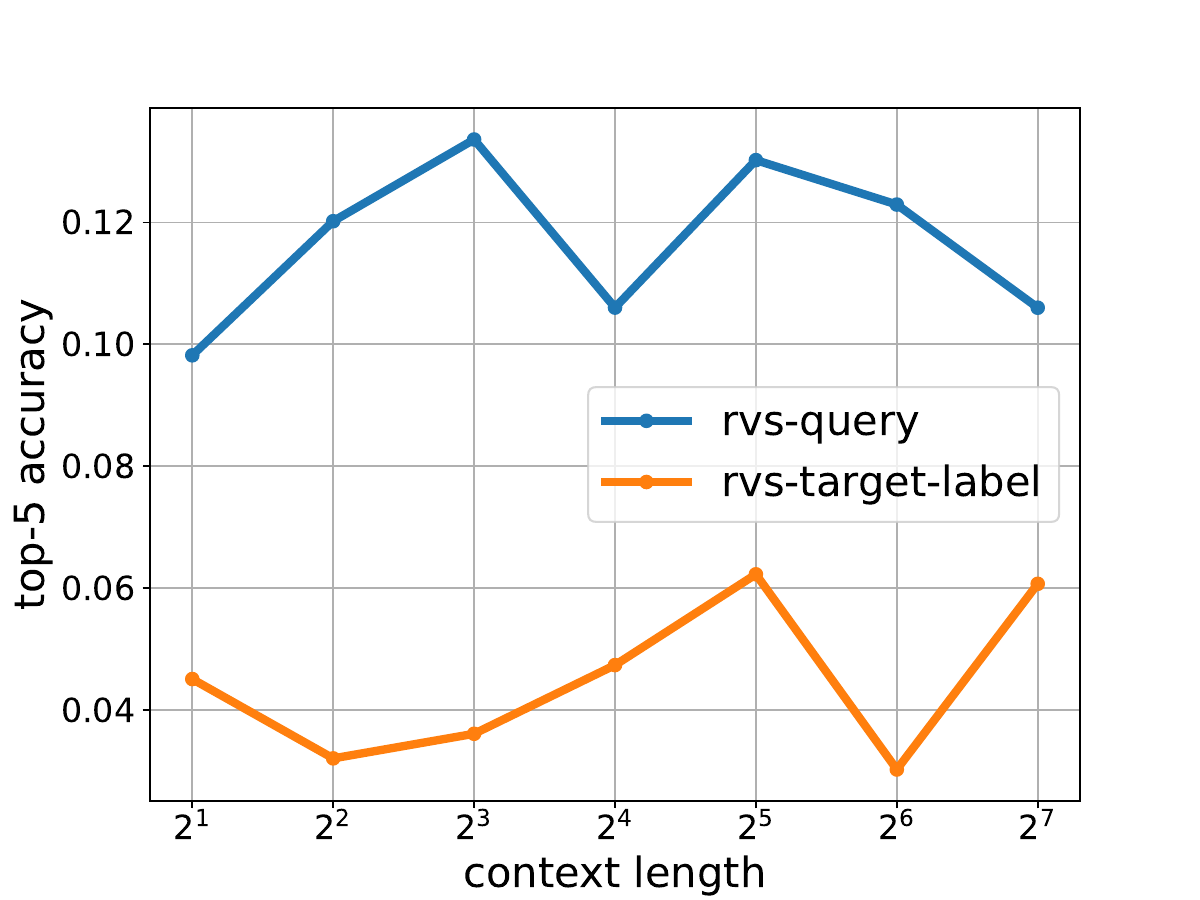}
        \vspace{-6pt}
        \label{fig:real-LLM-reverse3}
        \caption{English-French}
    \end{subfigure}
    \begin{subfigure}{0.24\textwidth}
        \centering
        \includegraphics[width=\textwidth]{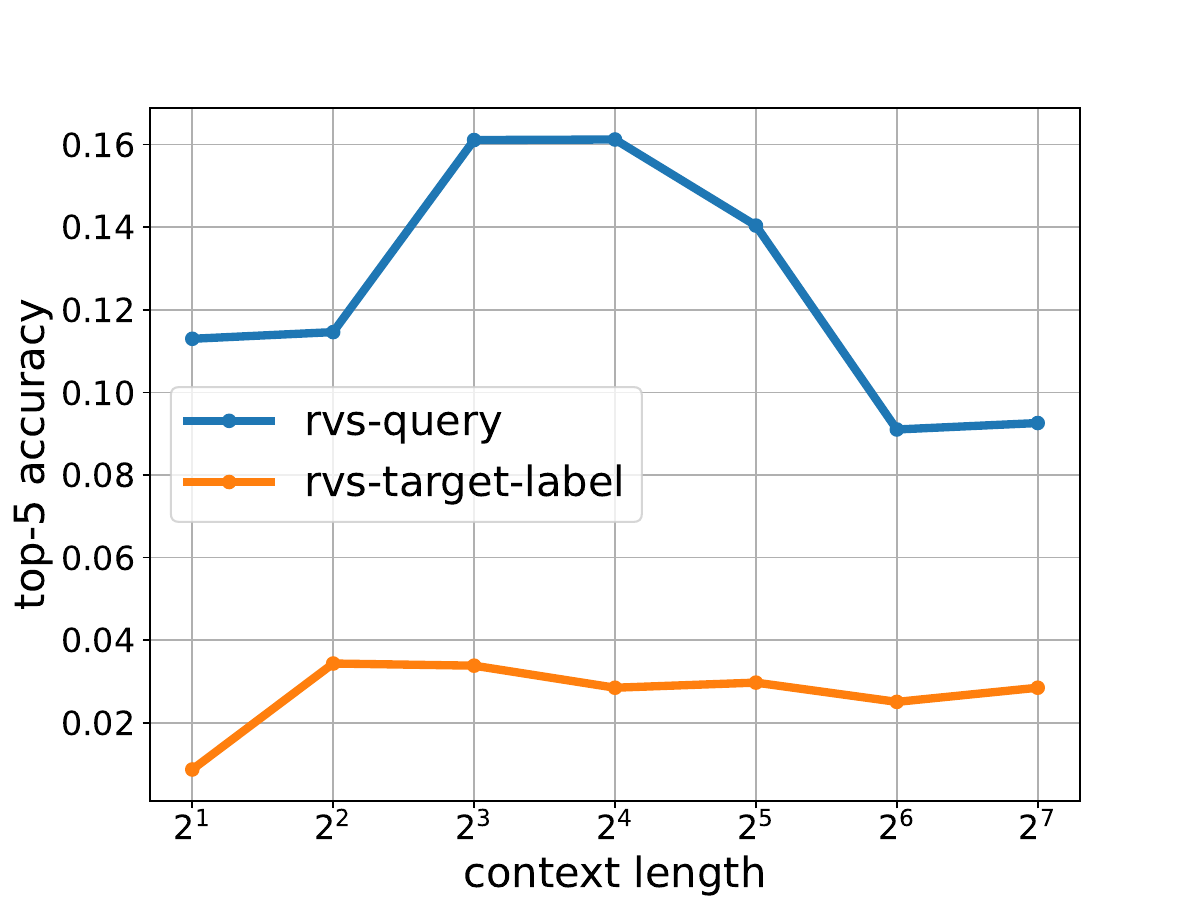}
        \vspace{-6pt}
        \label{fig:real-LLM-reverse4}
        \caption{English-German}
    \end{subfigure}
    \caption{The top-1 accuracy of predicting the reversed query word (\textcolor[RGB]{31,119,180}{blue}) and predicting the reversed target label word (\textcolor[RGB]{255,127,14}{orange}). The accuracy of predicting the reversed query word is higher than outputting the reversed target, indicating ICL makes ID predictions. }
    \label{fig:real-LLM-reverse}
\end{figure}

\begin{takeawaybox}[title=Summary of the Empirical Results \& Connections with the Existing Theories (I)]
Given an OOD context, ICL finds a near-optimal solution within its pretraining task space. Particularly, when learning OOD mathematical functions, ICL behaves as a predictor of its pretraining function class optimized by GD using the in-context examples. This validates and extends existing results by \cite{zhang2023trained} which theoretically shows that linear attention models trained on linear regression data still implement linear regression given arbitrary downstream context (see Appendix \ref{app:theo-detail-imple-LSA}). 
\end{takeawaybox}

Inspired by \cite{raventos2024pretraining}, we also explore whether increasing the diversity of training tasks, while keeping them in-distribution, can activate the OOD generalization ability. The results in Appendix \ref{sec:task-diversity} also suggest a negative conclusion.

\section{Learning Abstract Labels May Not Be A Real OOD Capability}
\label{sec:abstract-label-learning}
\subsection{Classification Tasks with Unseen Abstract Labels}
\label{subsec:abstract-label-learning}
Recent works \citep{pan2023context, kossen2024context} have shown that LLMs can perform classification tasks in which the labels are ``abstract symbols" with no semantic meaning. For instance, in the SST-2 binary classification task, the labels ``positive" and ``negative" are substituted with abstract terms like ``foo" and ``bar", respectively. These tasks are likely not seen during pretraining.  \citet{pan2023context} refer to this ability of ICL to perform such classification as ``task learning" (TL). In this section, we explore whether the TL ability really reflects a new-task-learning capability of ICL or if it merely stems from the model having learned similar tasks during pretraining. 

\textbf{The retrieval ability can be gained by pretraining on a retrieval task with diverse input-label mappings.} The classification of abstract labels can be approached by retrieving an example with semantics similar to the query and then outputting the label of that example. Therefore, the retrieval ability is a crucial prerequisite for performing abstract-label classification. We design a retrieval task to investigate whether ICL's retrieval capability can emerge from training on similar tasks. Specifically, we generate a predefined word embedding $E\in \mathbb{R}^{N \times d}$ and randomly sample $\boldsymbol{x}_i\in \mathbb{R}^d$ from the first 5 rows of $E$. Each vector $\boldsymbol{x}_i$ corresponds to the $I_{\boldsymbol{x}_i}$-th row of $E$, i.e., $\boldsymbol{x}_i = E_{I_{\boldsymbol{x}_i}}$. To generate the labels $y_i$, we follow these steps: First, map the index $I_{\boldsymbol{x}_i}$ to a new one $I_{y_i}\in [N]$ using the mapping rule $I_{y_i}=I_{\boldsymbol{x}_i}+s$, where $s\in \mathbb{N}$ is randomly sampled. Second, we set $y_i=E_{I_{y_i}}$. All in-context examples in a sequence share the same mapping rule defined by $s$. To succeed in this task, the model must retrieve the same token as the query example from the context and output its subsequent token. All models are trained with 200,000$\times$64 sequences,  where 200,000 is the number of training steps and 64 is the batch size. 

We train three models with three different ranges of $s$: $s\sim \mathcal{U}(50, 150)$, $s\sim \mathcal{U}(50, 250)$, and $s\sim \mathcal{U}(50, 450)$ and evaluate on $s\sim \mathcal{U}(50, 150)$, $s\sim \mathcal{U}(10, 20)$, and $s\sim \mathcal{U}(500, 600)$, where $\mathcal{U}$ denotes the uniform distribution. We plot the test error in Figure \ref{fig:retrieval}.

\begin{figure}[htbp]
    \centering
    \begin{subfigure}{0.26\textwidth}
        \centering
        \includegraphics[width=\textwidth]{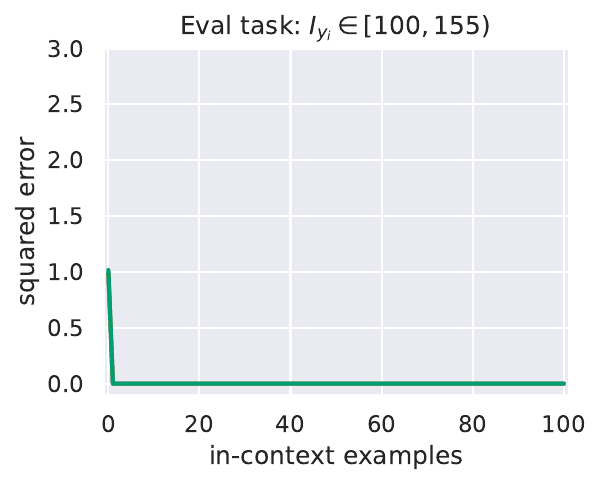}
        \vspace{-6pt}
        \label{fig:abstract-label-learning-sub1}
        \caption{Eval $s\sim \mathcal{U}(50, 150)$}
    \end{subfigure}
    \begin{subfigure}{0.26\textwidth}
        \centering
        \includegraphics[width=\textwidth]{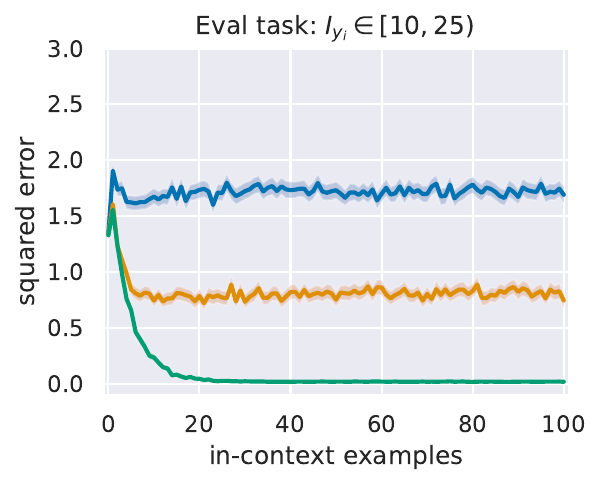}
        \vspace{-6pt}
        \label{fig:abstract-label-learning-sub2}
        \caption{Eval $s\sim \mathcal{U}(10, 20)$}
    \end{subfigure}
    \begin{subfigure}{0.37\textwidth}
        \centering
        \includegraphics[width=\textwidth]{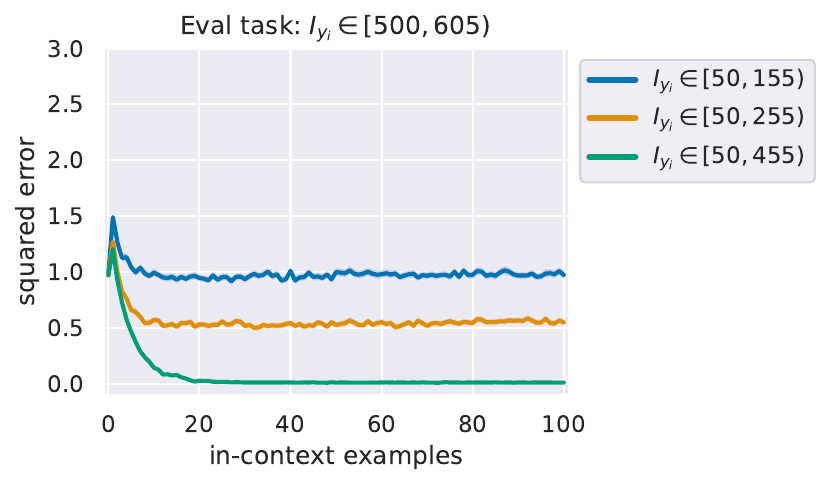}
        \vspace{-6pt}
        \label{fig:abstract-label-learning-sub3}
        \caption{Eval $s\sim \mathcal{U}(500, 600)$}
    \end{subfigure}
    \caption{The ICL test error of Transformers trained on the retrieval task with different numbers of label tokens. ``Eval" denotes ``evaluated on". Note that the indices of training label tokens $I_{y_i}\in [50, 455)$, so the labels in (a) are ID while (b) and (c) are OOD. }
    \label{fig:retrieval}
\end{figure}

\textbf{Observations.} In Figure \ref{fig:retrieval}, all three models perform well when labels are in distribution (a).  When the labels are OOD, the ICL performance improves with the number of label vectors (random mappings) encountered during training. This demonstrates that retrieval ability can emerge from training on diverse retrieval tasks. These findings may also offer new insights into how real-world LLMs develop in-context retrieval capabilities: when autoregressive pretraining includes numerous instances requiring the model to retrieve tokens from previous contexts, such abilities can emerge. We further validate this finding by observing the transition of the attention maps in Appendix \ref{app:prefix-matching-score}.

\textbf{The ability to perform linear regression and then retrieval can also be gained by pretraining on a similar task.} 
To further reproduce the emergence of the abstract label learning ability of real-world LLMs, we design a task that emulates the natural language classification with abstract labels like ``foo" and ``bar". The task function is defined as follows: $y_i=f(\boldsymbol{x}_i)=E_{I_{\boldsymbol{x}_i}}$, where $I_{\boldsymbol{x}_i}=\text{floor}(0.4*(\boldsymbol{w}^\top \boldsymbol{x}_i))+s$, with $E$ being the predefined word embedding and $s\in \mathbb{N}+$ shared in the same sequence. Here, $\boldsymbol{x}_i$, $\boldsymbol{w}\sim \mathcal{N}(0, 1)\in \mathbb{R}^d$. Each $\boldsymbol{x}_i$ is mapped to $y_i$ according to $\boldsymbol{w}$ and $s$. \footnote{In our experimental setup, given a sufficiently long context ($\approx 50$), the label of the query is highly likely to appear in the context, as the number of the possible classes is far less than the number of in-context examples.} In this task, estimating $\boldsymbol{w}$ and calculating $\boldsymbol{w}^\top \boldsymbol{x}_i$ simulates predicting the original label (``positive" and ``negative") based on the semantics in the natural language task, while retrieving the abstract labels from in-context examples that share the same $\text{floor}(0.4*(\boldsymbol{w}^\top \boldsymbol{x}_i))$ as the query from the context resembles identifying the abstract labels (``foo" and ``bar"). 

Again, we train three models on different ranges of mappings:  $s\sim \mathcal{U}(100, 200)$, $s\sim \mathcal{U}(100, 1000)$, and $s\sim \mathcal{U}(100, 2000)$, and evaluate on $s\sim \mathcal{U}(100, 200)$, $s\sim \mathcal{U}(500, 600)$, and $s\sim \mathcal{U}(3000, 3100)$. The test error is plotted in Figure \ref{fig:linear-regression-retrieval}.

\begin{figure}[htbp]
    \centering
    \begin{subfigure}{0.26\textwidth}
        \centering
        \includegraphics[width=\textwidth]{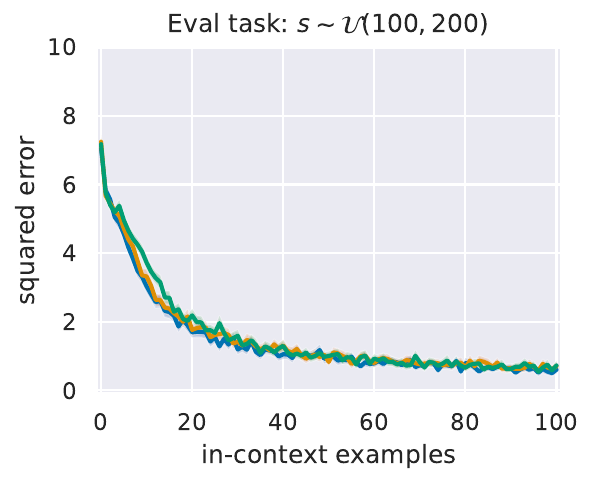}
        \caption{Eval $s\sim \mathcal{U}(100, 200)$}
        \label{fig:linear-regression-retrieval-sub1}
    \end{subfigure}
    \begin{subfigure}{0.26\textwidth}
        \centering
        \includegraphics[width=\textwidth]{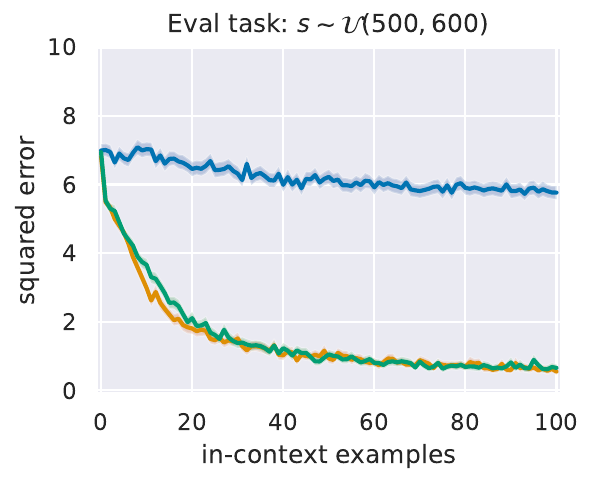}
        \caption{Eval $s\sim \mathcal{U}(500, 600)$}
        \label{fig:linear-regression-retrieval-sub2}
    \end{subfigure}
    \begin{subfigure}{0.37\textwidth}
        \centering
        \includegraphics[width=\textwidth]{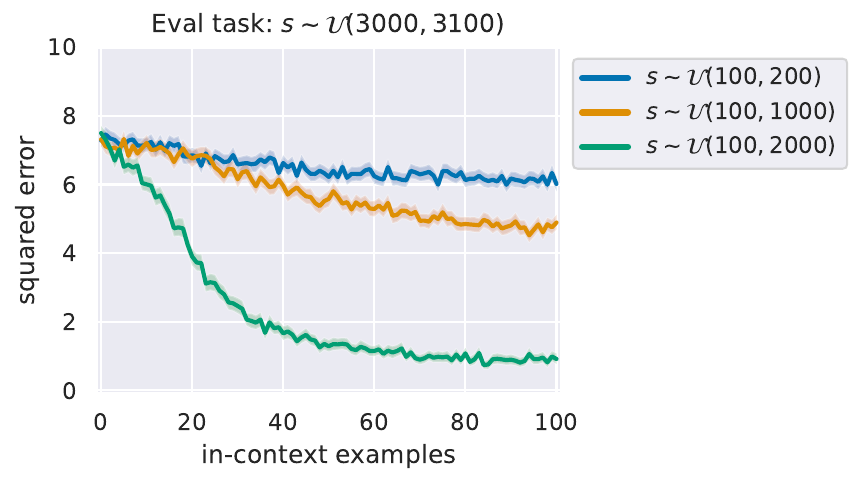}
        \caption{Eval $s\sim \mathcal{U}(3000, 3100)$}
        \label{fig:linear-regression-retrieval-sub3}
    \end{subfigure}
    \caption{The ICL test error of Transformers trained and tested on the linear regression + retrieval task with different numbers of label tokens. ``Eval" denotes ``evaluated on". Only the model trained on the largest number of tasks exhibits generalization to unseen label tokens. }
    \label{fig:linear-regression-retrieval}
\end{figure}

\textbf{Observations.} In Figure \ref{fig:linear-regression-retrieval}, the generalization ability to unseen labels also improves as the number of labels encountered during training increases. Notably, only the model trained with $s\sim \mathcal{U}(100, 2000)$ performs well on the unseen labels. This suggests that as long as the LLM has been exposed to sufficiently many similar tasks during training, it can effectively classify arbitrary OOD labels retrievable from context through ICL.

\subsection{Abstract Label Classification Can Only Be Achieved on ID Tasks}
\label{subsec:abstract-label-learning-iid}
\textbf{A retrieval task with OOD testing functions \& observations.} Given the above observations, one might question whether, once the target labels appear in the context, ICL can generalize beyond the training function class by retrieving the target label from the context. To investigate this, we conduct the same predict-then-retrieval task as in Figure \ref{fig:linear-regression-retrieval} but replace the testing functions with quadratic regression while preserving linear regression as the pretraining task. The results in Figure \ref{fig:quadratic-regression-retrieval} show that the generalization doesn't improve with training on more ID functions. Combining observations from Figure \ref{fig:linear-regression-retrieval}, we conclude that ICL can only solve classification with unseen labels over ID test function classes. Once the underlying task function is OOD, ICL fails even if the target label appears in the context. This finding highlights a limitation in improving an LLM's performance through in-context examples. While providing examples with shared labels may seem helpful, this approach may fail if the underlying prediction rule is too OOD for the LLM to learn.

\begin{figure}[htbp]
    \centering
    \begin{subfigure}{0.26\textwidth}
        \centering
        \includegraphics[width=\textwidth]{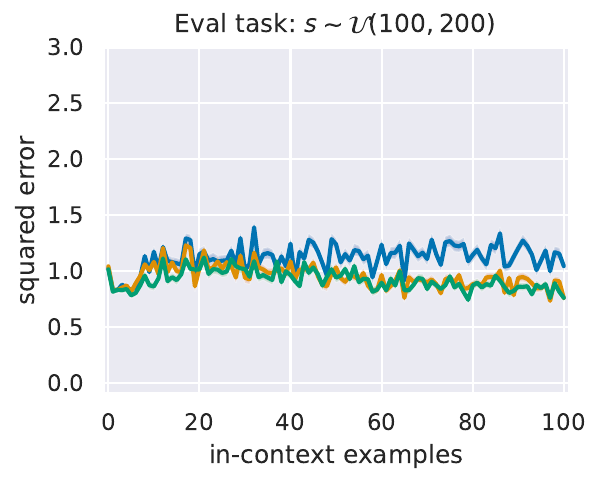}
        \caption{Eval $s\sim \mathcal{U}(100, 200)$}
        \label{fig:quadratic-regression-retrieval-sub1}
    \end{subfigure}
    \begin{subfigure}{0.26\textwidth}
        \centering
        \includegraphics[width=\textwidth]{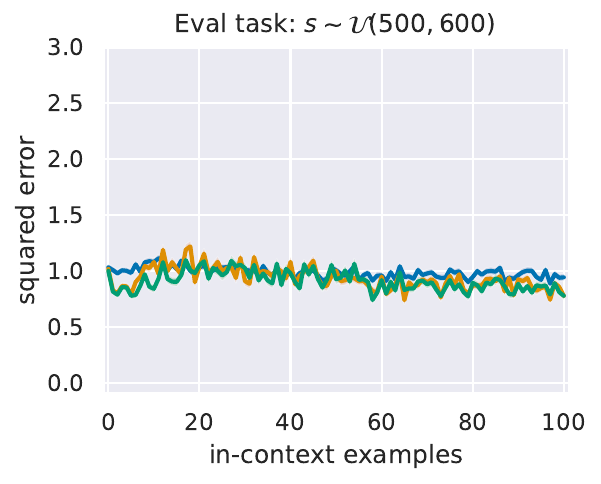}
        \caption{Eval $s\sim \mathcal{U}(500, 600)$}
        \label{fig:quadratic-regression-retrieval-sub2 }
    \end{subfigure}
    \begin{subfigure}{0.37\textwidth}
        \centering
        \includegraphics[width=\textwidth]{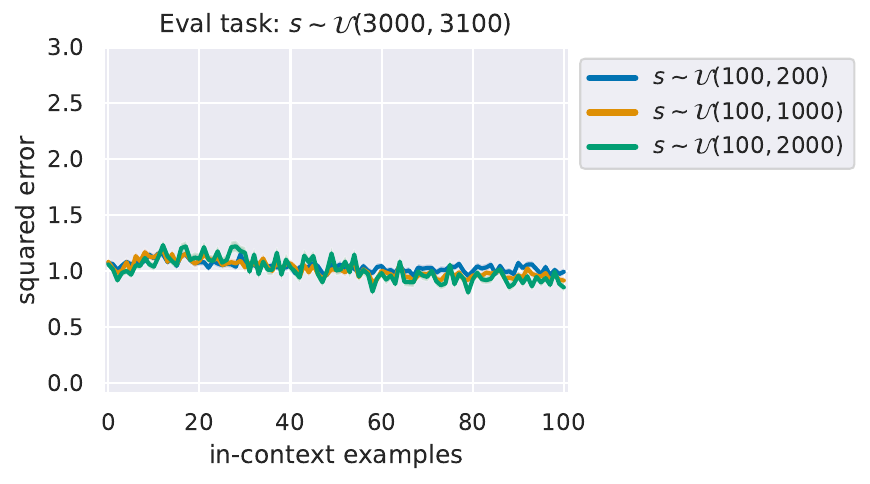}
        \caption{Eval $s\sim \mathcal{U}(3000, 3100)$}
        \label{fig:quadratic-regression-retrieval-sub3}
    \end{subfigure}
    \caption{The ICL test error of Transformers evaluated on a quadratic regression + retrieval task. Different colors denote models trained on the linear regression + retrieval task with different numbers of label tokens. ``Eval" denotes ``evaluation". The model trained on $s\sim \mathcal{U}(100, 2000)$ doesn't generalize better than the other two models.}
    \label{fig:quadratic-regression-retrieval}
    \vspace{-15pt}
\end{figure}

\begin{takeawaybox}[title=Summary of the Empirical Results \& Connections with the Existing Theories (II)]
To handle classification with abstract labels, the model infers an input-label mapping to implicitly establish a similarity metric. It then retrieves in-context examples similar to the query to deduce the OOD labels using this metric. However, this process succeeds only when the underlying function is ID, thus it does not represent a true OOD generalization capability. This observation aligns with Bayesian frameworks for ICL—the implicit similarity metric here corresponds to the task concept inferred by the model. We leave an intuitive Bayesian interpretation of the findings in Section \ref{subsec:abstract-label-learning} and \ref{subsec:abstract-label-learning-iid} in Appendix \ref{app:bayesian-interpretation}.
\end{takeawaybox}

\subsection{Real-world LLMs May Not Necessarily In-context Learn New Tasks}
\label{sec:llama-ood-task}

\textbf{Evaluating Llama-3 on an OOD synthetic word classification task.} Now we assess whether real-world LLMs can tackle OOD tasks through ICL. We design a synthetic word classification task for a pretrained Llama-3-8B. Specifically, we randomly sample $\boldsymbol{x}_i\in \mathbb{R}^{d}$ from the word embedding of Llama-3-8B (denoted as $E_{llama}$) and generate random linear mappings $\boldsymbol{W}\in \mathbb{R}^{d\times C}$ as task functions (where $C=10$). The label words are created by mapping $\boldsymbol{x}_i$ to one of the ten label vectors in $E_{llama}$ using $\boldsymbol{W}$. Experimental details are in Appendix \ref{app:llama-ood-task}. To complete this task, the model must learn $\boldsymbol{W}$ in context, which is unlikely to have been seen during the pretraining of Llama.

For comparison, we also evaluate the ICL performance of Llama-3-8B on a retrieval version of this task.  Concretely, we first randomly sample $C=10$ different vectors from $E_{llama}$ as $\boldsymbol{x}_i$ and compute $\boldsymbol{y}_i$ in the same way as the above classification task to get $S=[\boldsymbol{x}_1, \boldsymbol{y}_1,...,\boldsymbol{x}_C, \boldsymbol{y}_C]$. Then we repeat $S$ 20 times to construct the input sequence $S'=[S\oplus S\oplus...\oplus S]$, where $\oplus$ denotes concatenation operation. The goal is to predict the next token given a prefix of $S'$. To succeed in this task, the model has to retrieve the same token as the query token (the last $\boldsymbol{x}_i$ of $S'$) from the context and output its subsequent token $\boldsymbol{y}_i$. The results of these two tasks are presented in Figure \ref{fig:word-classification-llama-3}.

\textbf{Observations.} From Figure \ref{fig:word-classification-llama-3}, we observe that the ICL performance on the synthetic classification task is close to random guessing (10\% accuracy), while performance on the retrieval task is significantly better (similar results also hold for Llama-2-7B in Appendix \ref{app:word-classification-llama-2}). This suggests that pretrained real-world LLMs may also struggle to learn new input-output mappings from context; instead, ICL appears to be more adept at retrieval tasks. To show that the failure in the synthetic word classification task is mainly due to its OOD nature instead of some other factors that make it difficult to learn, we train a GPT-2 to perform the same task in Appendix \ref{app:word-classification-easy} and find that the task can be well addressed after training. 

\begin{figure}[htbp]
    \centering
    \includegraphics[width=0.33\textwidth]{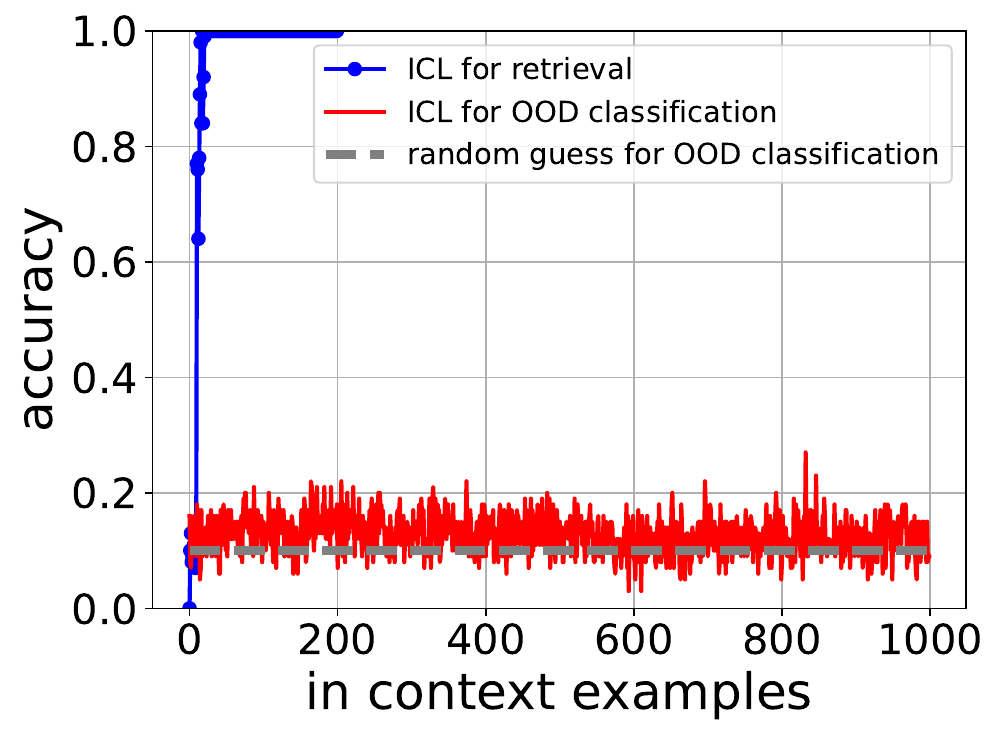}
    \caption{The ICL accuracy of Llama-3-8B on the synthetic tasks.}
    \label{fig:word-classification-llama-3}
    \vspace{-5pt}
  \end{figure}

\section{The Algorithm Selection Mechanism Exists Broadly When Evaluated on OOD Tasks}
\label{sec:algo-selection}

Real-world LLMs are pretrained on a huge corpus that could contain massive tasks. \citet{bai2024transformers, yadlowsky2023pretraining} have empirically found that the ICL performance of Transformers trained on multiple tasks approaches the optimal pretraining function when evaluated on one of the training tasks. In this section, we will show that this algorithm-selection phenomenon of ICL persists even when evaluated on OOD tasks, regardless of the distribution of the testing functions, and provide a theoretical characterization of the algorithm-selection mechanism.

\textbf{The Model pretrained on a single task vs. the model pretrained on multiple tasks.} In Figure \ref{fig:algo-selection}, we compare the performance of GPT-2 models trained on a single task—linear regression (LR), quadratic regression (QR), 2-layer ReLU network (ReLU NN) regression—against the model trained on all three tasks when encountering four kinds of OOD tasks. We also plot the error of a 2-layer ReLU NN trained by GD (dashed blue line). The results are in Figure \ref{fig:algo-selection}.

\begin{figure}[htbp]
    \centering
    \begin{subfigure}{0.2\textwidth}
        \centering
        \includegraphics[width=\textwidth]{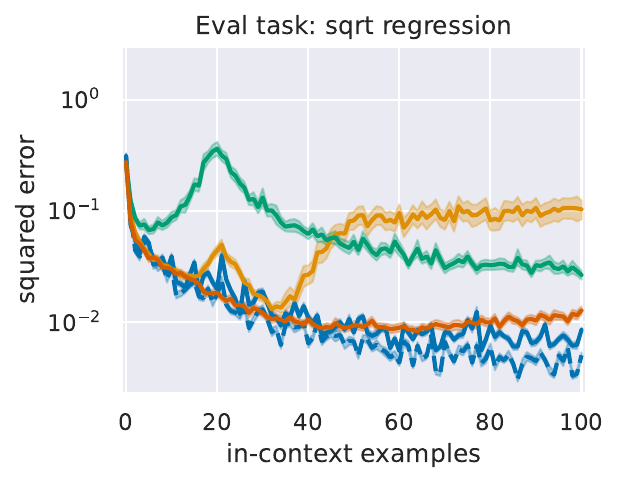}
        \label{fig:algo-selection-1}
        \vspace{-8pt}
        \caption{}
        \vspace{-5pt}
    \end{subfigure}
    \begin{subfigure}{0.2\textwidth}
        \centering
        \includegraphics[width=\textwidth]{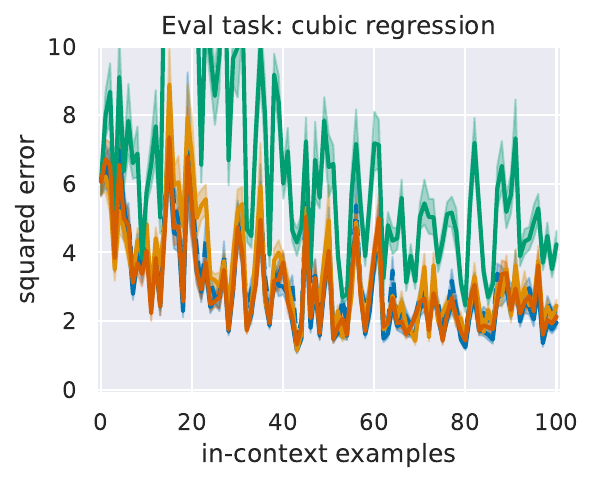}
        \label{fig:algo-selection-2}
        \vspace{-8pt}
        \caption{}
        \vspace{-5pt}
    \end{subfigure}
    \begin{subfigure}{0.2\textwidth}
        \centering
        \includegraphics[width=\textwidth]{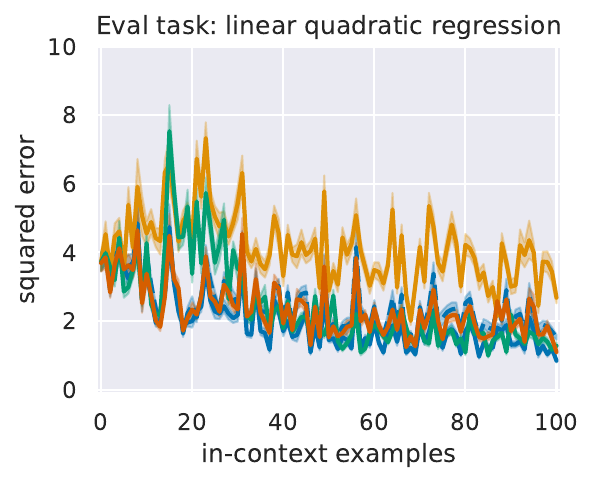}
        \label{fig:algo-selection-3}
        \vspace{-8pt}
        \caption{}
        \vspace{-5pt}
    \end{subfigure}
    \begin{subfigure}{0.3\textwidth}
        \centering
        \includegraphics[width=\textwidth]{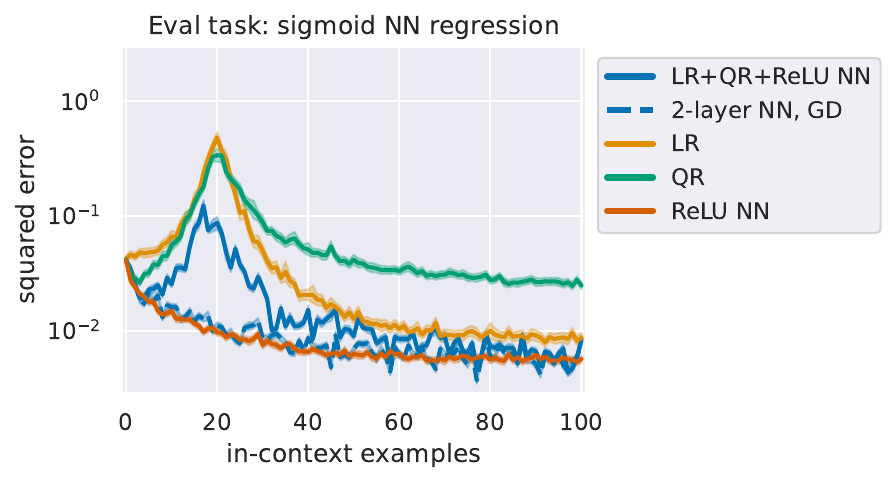}
        \label{fig:algo-selection-4}
        \vspace{-8pt}
        \caption{}
        \vspace{-5pt}
    \end{subfigure}
    \caption{The ICL performance of models trained on the individual task: linear regression (LR), quadratic regression (QR), 2-layer ReLU network (ReLU NN) regression, and the model trained on the mixture of the three tasks (LR+QR+ReLU NN). The evaluation functions are (a) square root, (b) cubic, (c) linear+quadratic, and (d) 2-layer Sigmoid network (details in Appendix \ref{app:exp-detail-1}). The performance of the model trained on the mixed tasks is comparable to that of the model trained on the single task that performs the best on the evaluation task.}
    \label{fig:algo-selection}
\end{figure}

\textbf{Observations.} 1) the ICL performance of the model trained on mixed tasks (LR+QR+ReLU NN) is comparable to the performance of the model trained on a single task with the lowest test error on the evaluation task. This suggests that ICL can automatically select the best pretraining functions according to the downstream context. 2) ReLU NN consistently performs the best on all four OOD test functions. Moreover, the performance of the ReLU model trained by GD aligns well with the ICL performance of the GPT-2 trained on the same function class. This demonstrates that our findings in Section \ref{subsec:main} still hold when the Transformer is trained on a mixture of multiple tasks.

\subsection{Theoretically Revealing the Mechanism of Algorithm Selection}
\label{sec:theory}
In this section, we will provide theoretical insight into the working mechanism of the algorithm selection of ICL. We find there simultaneously exist two parallel mechanisms: the \textbf{Low-test-error preference} and the \textbf{Similar-input-distribution preference.}

\label{sec:theo-algo-selection}
\textbf{A mixed Gaussian pretraining dataset of multiple tasks.} We adopt the theoretical framework by \cite{lin2024dual}. Consider a noisy linear regression pretraining dataset with the inputs and task weights following the mixed Gaussian distribution:
\begin{assumption}
    \label{ass:mixed-gaussian}
    (Mixed Gaussian pretraining data) Input distribution: $P(\boldsymbol{x}| \boldsymbol{\mu})=\mathcal{N}(\boldsymbol{x}|\boldsymbol{\mu},\sigma_x^2 \boldsymbol{I})$, label distribution: $ P(y|\boldsymbol{x}, \boldsymbol{w})=\mathcal{N}\left(y|\langle\boldsymbol{x}, \boldsymbol{w}\rangle, \sigma_y^2\right)$. The input means and task weights are sampled from a mixed Gaussian distribution: $P(\boldsymbol{\mu}, \boldsymbol{w})=\sum_{m=1}^M \pi_m \mathcal{N}(\boldsymbol{\mu} ; \boldsymbol{\mu}_m, \sigma_\mu^2 \boldsymbol{I}) \cdot \mathcal{N}\left(\boldsymbol{w} ; \boldsymbol{w}_m, \sigma_w^2 \boldsymbol{I}\right)$, where $\sum_{m=1}^M \pi_m=1$, $0<\pi_m<1$ and $\left\|\boldsymbol{\mu}_m\right\|=\left\|\boldsymbol{w}_m\right\|=1, \forall m$. Define $\delta_{\mu} = \frac{\sigma^2_{\mu}}{\sigma^2_{x}}$ and $\delta_{w} = \frac{\sigma^2_{w}}{\sigma^2_{y}}$. Each training sequence $\mathcal{S}_T=[\boldsymbol{x}_1 \oplus y_1 \oplus ... \oplus\boldsymbol{x}_T \oplus y_T]$ is constructed by first sampling the input mean and the task weight according to $P(\boldsymbol{\mu}, \boldsymbol{w})$ and then sampling $\boldsymbol{x}_i$ and $y_i$ according to $P(\boldsymbol{x}| \boldsymbol{\mu})$ and $P(y|\boldsymbol{x}, \boldsymbol{w})$, respectively. Denote this pretraining distribution as $P_{tr}$.
\end{assumption}

According to the Corollary 2 of \cite{lin2024dual} (see \ref{lem:mixed-gaussian-pred} in Appendix \ref{proof:algo-selection}), the closed-form prediction of the model trained on the pretraining data under Assumption \ref{ass:mixed-gaussian}, given the testing context, remains a Gaussian mixture of the reweighted pretraining task weights: $\left\langle \boldsymbol{x}_{i+1}, \sum_{m=1}^M \tilde{\pi}_m \tilde{\boldsymbol{w}}_m\right\rangle$, where $\tilde{\pi}_m$ and $\tilde{\boldsymbol{w}}_m$ are the posterior variables of $\pi_m$ and $\boldsymbol{w}_m$ given the downstream context. Hence, to analyze how ICL selects pretraining priors, the key lies in uncovering how $\tilde{\pi}_m$ evolves. First, we introduce Lemma \ref{lem:dual-x-shift} that characterizes the ratio of the reweighted weight of two pretraining tasks:

\begin{lemma}
    \label{lem:dual-x-shift}
    (Appendix H.1 of \citet{lin2024dual}) Consider any two different pretraining component $\alpha$ and $\beta$, given a testing context $\mathcal{S}_T \oplus \boldsymbol{x}_{T+1}$ and the well-pretrained model $M^*$, the ratio between the weights of the two task priors $\tilde{\pi}_\alpha/\tilde{\pi}_\beta$ in $M^*$'s ICL prediction can be decomposed into two terms: $\tilde{\pi}_\alpha/\tilde{\pi}_\beta = \frac{\pi_\alpha}{\pi_\beta} \exp \left( \Psi_{\boldsymbol{\mu}}(\alpha, \beta) + \Psi_{\boldsymbol{w}}(\alpha, \beta)\right) $, where 
    \begin{equation}
        \begin{aligned}
                &\Psi_{\boldsymbol{\mu}}(\alpha, \beta)=\left(\sum_{i=1}^{T+1}\left\|\boldsymbol{\mu}_\beta-\boldsymbol{x}_i\right\|^2-\sum_{i=1}^{T+1}\left\|\boldsymbol{\mu}_\alpha-\boldsymbol{x}_i\right\|^2\right) /\left(2 \sigma_x^2\left(1+(T+1) \delta_\mu\right)\right).\\
        \end{aligned}
    \end{equation}
    Further, assuming the testing in-context examples $\boldsymbol{x}_i\sim \mathcal{N}(\boldsymbol{\mu}^*, \tau_x^2 \boldsymbol{I})$, if $\|\boldsymbol{\mu}_\beta - \boldsymbol{\mu}^*\|^2 - \|\boldsymbol{\mu}_\alpha - \boldsymbol{\mu}^*\|^2 \geq 0$ holds, then as the context length $T \rightarrow \infty$,  the first term $\Psi_{\boldsymbol{\mu}}(\alpha, \beta)\rightarrow (\|\boldsymbol{\mu}_\beta - \boldsymbol{\mu}^*\|^2 - \|\boldsymbol{\mu}_\alpha - \boldsymbol{\mu}^*\|^2)/2 \sigma^2_\mu \geq 0$.
\end{lemma}

However, \citet{lin2024dual} didn't analyze how the second term $\Psi_{\boldsymbol{w}}(\alpha, \beta)$ would evolve given any downstream task, which we will demonstrate to play an important role in the algorithm selection mechanism. In the following theorem, we prove that $\Psi_{\boldsymbol{w}}(\alpha, \beta)$ converges to a non-negative value when the pretraining function class $\alpha$ performs better on the downstream context than $\beta$.

\begin{theorem}
    \label{theo:algo-selection}
     (ICL prediction favors the pretraining function with low error on the context) Given the context $\boldsymbol{S}_T$, if the empirical risk of implementing a function of the pretraining task $\alpha$ is less than that of $\beta$, i.e., $\frac{1}{T} \sum_{i=1}^{T} |\boldsymbol{w}_\beta \boldsymbol{x}_i - y_i|^2 - |\boldsymbol{w}_\alpha \boldsymbol{x}_i - y_i|^2 \geq 0$, 
    then, under some mild Assumptions \ref{ass:theo:algo-selection-context} on the distribution of $\boldsymbol{S}_T$, we have $\Psi_{\boldsymbol{w}}(\alpha, \beta) \geq 0$. 
    
    Combining Lemma \ref{lem:dual-x-shift}, if the downstream inputs $\boldsymbol{x}_i$, $\boldsymbol{x}_i\sim \mathcal{N}(\boldsymbol{\mu}^*, \tau_x^2 \boldsymbol{I})$ and $\|\boldsymbol{\mu}_\beta - \boldsymbol{\mu}^*\|^2 - \|\boldsymbol{\mu}_\alpha - \boldsymbol{\mu}^*\|^2 \geq 0$ hold, then as $T\rightarrow \infty$, we have $\tilde{\pi}_\alpha/\tilde{\pi}_\beta \geq \pi_\alpha/\pi_\beta$.
\end{theorem}

\textbf{Summary of the algorithm-selection mechanism.} 
\ref{lem:dual-x-shift} and Theorem \ref{theo:algo-selection} together elucidate the algorithm-selection mechanism of ICL. According to Lemma \ref{lem:mixed-gaussian-pred}, the ICL prediction of the model pretrained on the mixed Gaussian data will be a reweighted combination of the pretraining task vectors $\boldsymbol{w}_i$. Whether the ratio between the weights of two pretraining tasks, $\tilde{\pi}_\alpha/\tilde{\pi}_\beta$, given a downstream context, exceeds the original ratio $\pi_\alpha/\pi_\beta$ depends on two factors: 1) whether the pretraining input distribution of $\alpha$ is closer to the downstream input distribution than that of $\beta$; 2) whether the task function of $\alpha$ induces lower test error in downstream context than that of $\beta$. When both conditions are met, we have $\tilde{\pi}_\alpha/\tilde{\pi}_\beta \geq \pi_\alpha/\pi_\beta$, indicating that ICL prefers $\alpha$ over $\beta$ in its predictions. We leave the discussions of the advantage of our theory result in Appendix \ref{app:discussion-of-our-theory} and offer an intuitive Bayesian interpretation of the algorithm selection in Appendix \ref{app:bayesian-interpretation}.

\subsection{Empirical Validation of the Algorithm-selection Mechanism of ICL}
\label{sec:exp-algo-selection}
Now we validate our theoretical findings regarding ICL's algorithm-selection mechanism in OOD tasks by conducting numerical experiments following \citet{lin2024dual}. In Figure \ref{fig:algo-selection-mechanism1} and \ref{fig:algo-selection-mechanism2}, the training data is a Gaussian mixture with four components (see Assumption \ref{ass:mixed-gaussian}), while the test function is a two-layer ReLU network (Appendix \ref{app:exp-detail-1}). Both the training and the test data are in ICL format. We compute the test error of using each pretraining task function to predict the downstream function (the first row of Figure \ref{fig:algo-selection-mechanism1} and \ref{fig:algo-selection-mechanism2}), the weights for each pretraining function during ICL inference (the second row), and the test error of the pretrained ICL model with the closed form prediction derived in Lemma \ref{lem:mixed-gaussian-pred} (the third row). We evaluate five different noise levels ($\delta_x=\delta_w\in \{1/81, 9/1,1,9,81\}$) and consider two settings described below.

\textbf{Low-test-error preference of ICL.} To validate Theorem \ref{theo:algo-selection}, we ensure that the distributional distances between the inputs of each training task and the test data remain consistent. Specifically, all $\boldsymbol{x}_i$ in both training and test data are sampled from $\mathcal{N}([0,0,0]^\top,\sigma^2_x\boldsymbol{I})$. The task weights for different pretraining tasks vary, as detailed in the top half of Table \ref{table:algo-selection-exp-setting}. In this setup, only the test error of the pretraining functions influences algorithm selection. From the top two rows of Figure \ref{fig:algo-selection-mechanism1}, we can observe a clear negative correlation between the ICL performance and the test error of the task weight. This result supports Theorem \ref{theo:algo-selection}, confirming that ICL prefers the pretraining functions with a low test error in the downstream context. Also, it's consistent with the observations in Figure \ref{fig:algo-selection}.

\textbf{Similar-input-distribution preference of ICL.} We also empirically validate Lemma \ref{lem:dual-x-shift} in Figure \ref{fig:algo-selection-mechanism2}. In this case, the distributional distances between the input of different pretraining tasks and that of the test context vary: the distances of different tasks are ordered from largest to smallest as $1>2>3>4$, while the test errors of different pretraining functions are almost the same (detailed setup is in the bottom half of Table \ref{table:algo-selection-exp-setting}).  As shown in the bottom two rows in Figure \ref{fig:algo-selection-mechanism2}, the task weight $\tilde{\pi}_i$ is positively correlated with the similarity between the training and testing input distribution. This is consistent with Lemma \ref{lem:dual-x-shift} which demonstrates that ICL prefers to select the pretraining function whose input distribution is close to the downstream one.

\begin{figure}[htbp]
    \centering
    \begin{subfigure}{1\textwidth}
        \includegraphics[width=\textwidth]{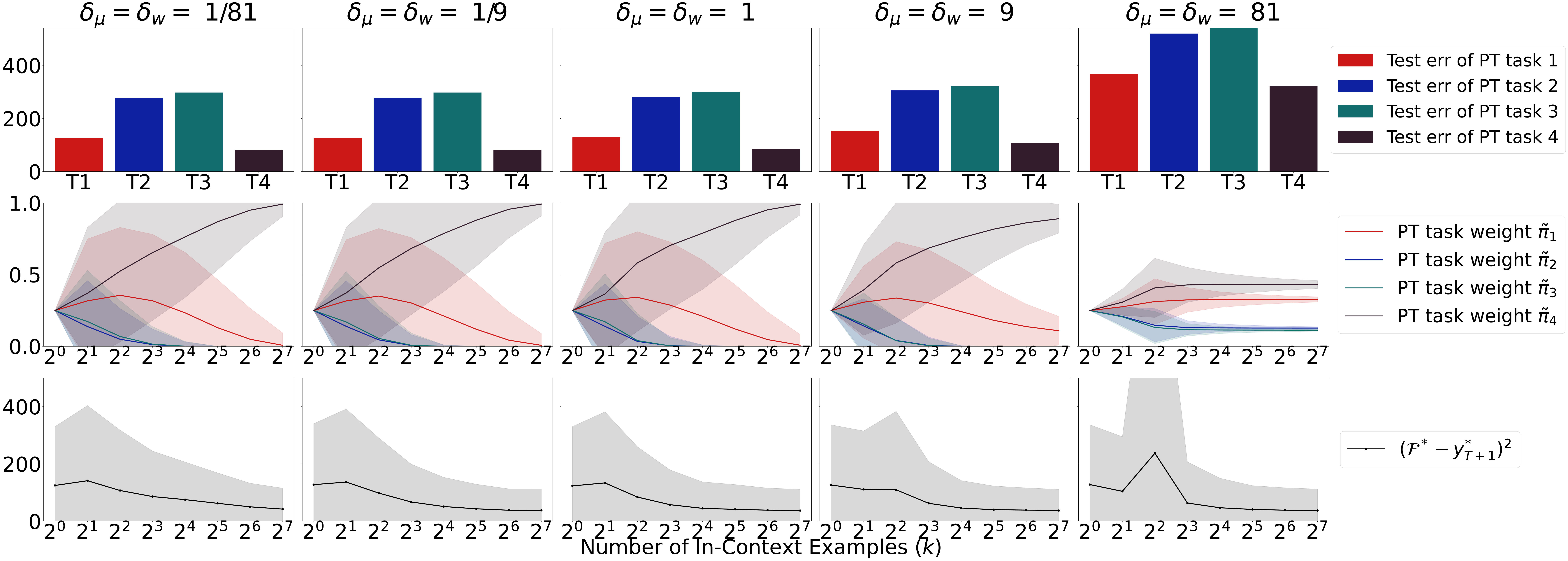}
        \vspace{-5pt}
        \caption{Numerical verification of the low-test-error preference}
        \label{fig:algo-selection-mechanism1}
    \end{subfigure}
    
    \begin{subfigure}{1\textwidth}
        \includegraphics[width=\textwidth]{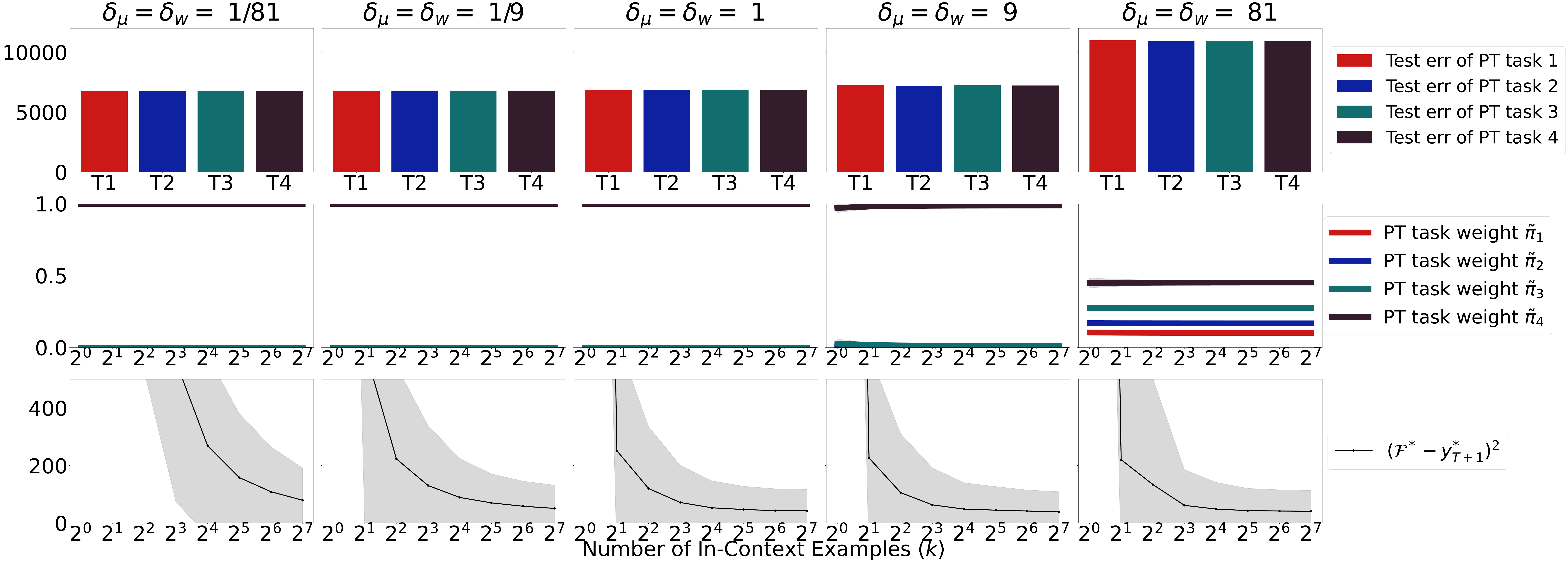}
        \vspace{-5pt}
        \caption{Numerical verification of the similar-input-distribution preference}
        \label{fig:algo-selection-mechanism2}
    \end{subfigure}
    \caption{Empirical validation of the algorithm-selection mechanism of ICL. The first rows: the test error of the four pretraining functions. The mid rows: the weights of each pretraining function in the closed-form downstream ICL prediction (given by Lemma \ref{lem:mixed-gaussian-pred}). The last rows: the test error of the pretrained ICL model with the closed form prediction derived in Lemma \ref{lem:mixed-gaussian-pred}. \textbf{Observations.} 1) In the first two rows of Figure \ref{fig:algo-selection-mechanism1}, the value of the task weight $\tilde{\pi}_i$ is negatively correlated with the test error of pretraining task $i$. 2) In the first two rows of Figure \ref{fig:algo-selection-mechanism2} the task weights are negatively correlated with the distance between the training and testing input distribution.}
    \label{fig:algo-selection-mechanism}
\end{figure}

\begin{table}[htbp]
    \centering
    \caption{Experiment setting of Figure \ref{fig:algo-selection-mechanism1} and Figure \ref{fig:algo-selection-mechanism2}. ``PT" and ``DS" are short for ``pretraining" and ``downstream", respectively.}
    \resizebox{\linewidth}{!}{
        \begin{tabular}{cccccc}
            \toprule
            Experiment & DS inputs & PT task id & PT input distribution & PT task functions  & PT-DS input distance \\
            \midrule
           \multirow{4}{*}{Figure \ref{fig:algo-selection-mechanism1}}& \multirow{4}{*}{$\mathcal{N}([0,0,0]^\top,\sigma^2_x\boldsymbol{I})$}& 1& $\mathcal{N}([0,0,0]^\top,\sigma^2_x\boldsymbol{I})$ & $\mathcal{N}([5,5,5]^\top,\sigma^2_w\boldsymbol{I})$   & 0 \\
           & & 2& $\mathcal{N}([0,0,0]^\top,\sigma^2_x\boldsymbol{I})$ &$\mathcal{N}([-5,5,5]^\top,\sigma^2_w\boldsymbol{I})$   & 0  \\
           & & 3& $\mathcal{N}([0,0,0]^\top,\sigma^2_x\boldsymbol{I})$ &$\mathcal{N}([-5,5,-5]^\top,\sigma^2_w\boldsymbol{I})$   & 0  \\
           & & 4& $\mathcal{N}([0,0,0]^\top,\sigma^2_x\boldsymbol{I})$ &$\mathcal{N}([-5,-5,-5]^\top,\sigma^2_w\boldsymbol{I})$   & 0  \\
            \midrule
           \multirow{4}{*}{Figure \ref{fig:algo-selection-mechanism2}}& \multirow{4}{*}{$\mathcal{N}([-4,-4,-4]^\top,\sigma^2_x\boldsymbol{I})$}& 1& $\mathcal{N}([5,5,5]^\top,\sigma^2_w\boldsymbol{I})$ & $\mathcal{N}([1,1,1]^\top,\sigma^2_w\boldsymbol{I})$  & 15.59 \\
           & & 2&$\mathcal{N}([-5,5,5]^\top,\sigma^2_w\boldsymbol{I})$ &$\mathcal{N}([1,1,1]^\top,\sigma^2_w\boldsymbol{I})$  & 12.77  \\
           & & 3&$\mathcal{N}([-5,5,-5]^\top,\sigma^2_w\boldsymbol{I})$ &$\mathcal{N}([1,1,1]^\top,\sigma^2_w\boldsymbol{I})$  & 9.11  \\
           & & 4&$\mathcal{N}([-5,-5,-5]^\top,\sigma^2_w\boldsymbol{I})$ &$\mathcal{N}([1,1,1]^\top,\sigma^2_w\boldsymbol{I})$   & 1.73  \\
            \bottomrule
        \end{tabular}
    }
    \label{table:algo-selection-exp-setting}
\end{table}

\subsection{Verifying the Algorithm-selection Mechanism on Real-world LLMs}
\label{sec:realLLM-selection}
In this section, we investigate whether real-world LLMs can perform algorithm selection through ICL. To achieve this, we design an ambiguous sentence classification task, in which each sentence can be classified based on one of three aspects: ``sentiment", ``type", or ``location". For each ICL sequence, we select one of the aspects as the classification criterion and map the label words to meaningless strings. For instance, if we choose to classify each sentence according to its sentiment, then ``positive," ``neutral," and ``negative" are mapped to ``RqF," ``IwZ," and ``SdK," respectively. Detailed experimental setups are in Appendix \ref{app:realLLM-selection}. We compute the top-5 accuracy of different classification criteria. The results in Figure \ref{fig:realLLM-selection} show that as the context length increases, the LLM finds the most appropriate criterion, exhibiting the low-test-error preference.

\begin{figure}[htbp]
\vspace{-10pt}
    \centering
    \begin{subfigure}{0.27\textwidth}
        \centering
        \includegraphics[width=\textwidth]{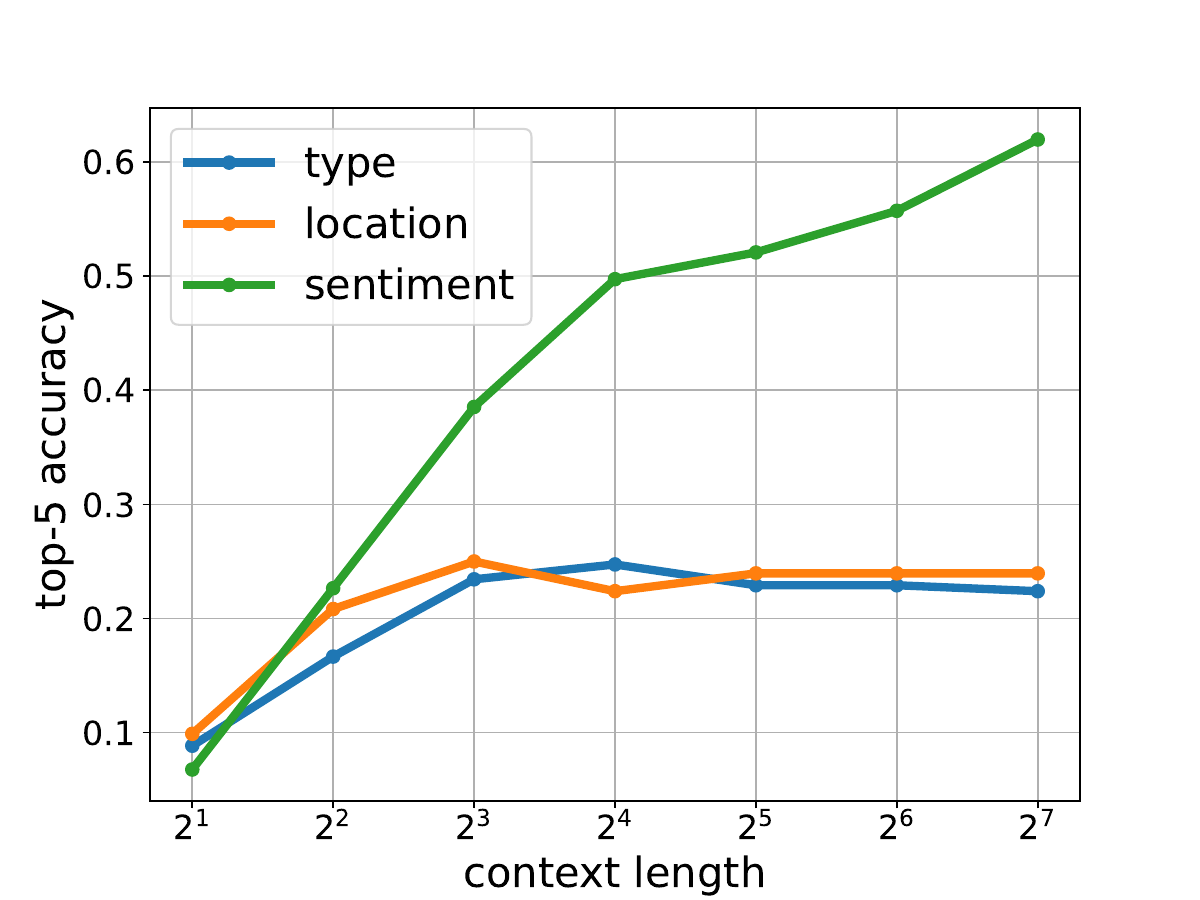}
        \vspace{-5pt}
        \label{fig:realLLM-selection1}
        \caption{Sentiment}
    \end{subfigure}
    \hspace{-5pt}
    \begin{subfigure}{0.27\textwidth}
        \centering
        \includegraphics[width=\textwidth]{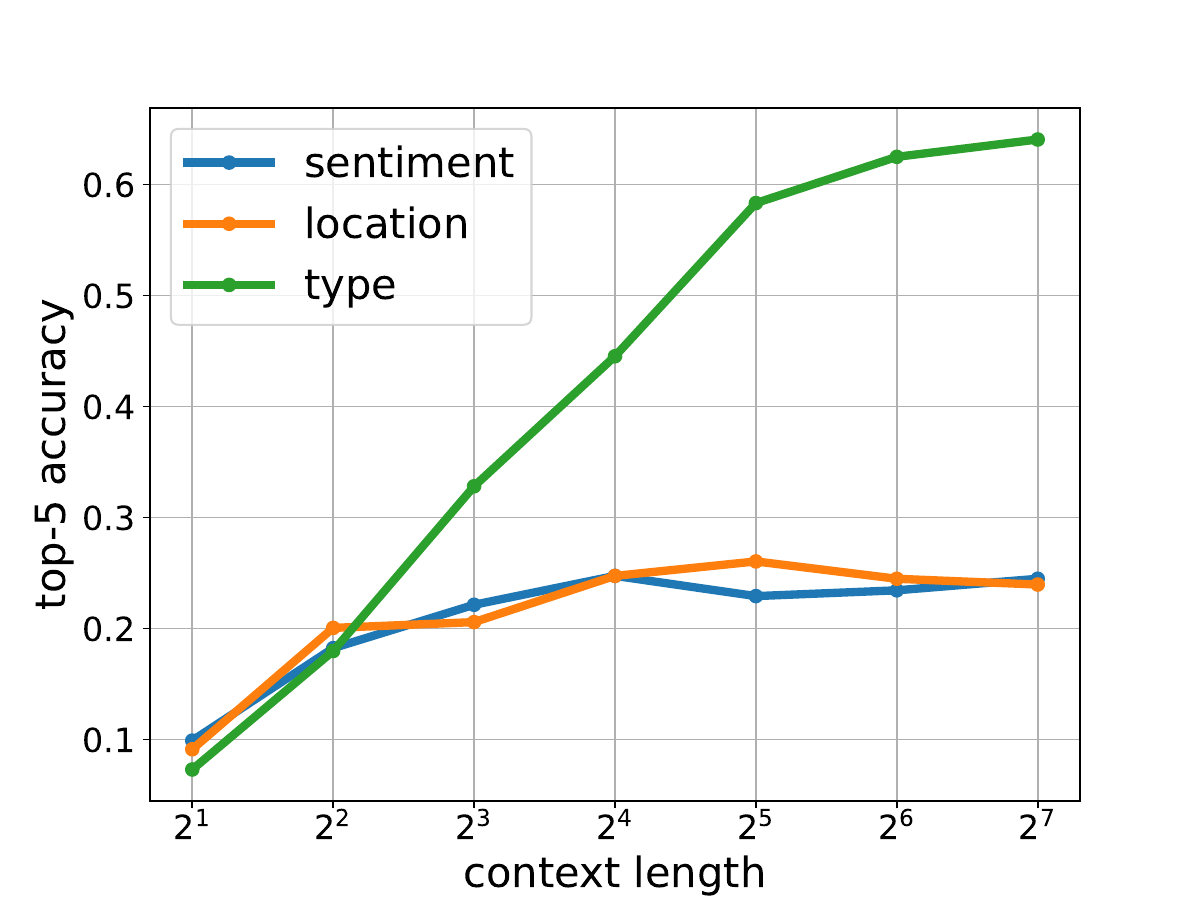}
        \vspace{-5pt}
        \label{fig:realLLM-selection2}
        \caption{Type}
    \end{subfigure}
    \hspace{-5pt}
    \begin{subfigure}{0.27\textwidth}
        \centering
        \includegraphics[width=\textwidth]{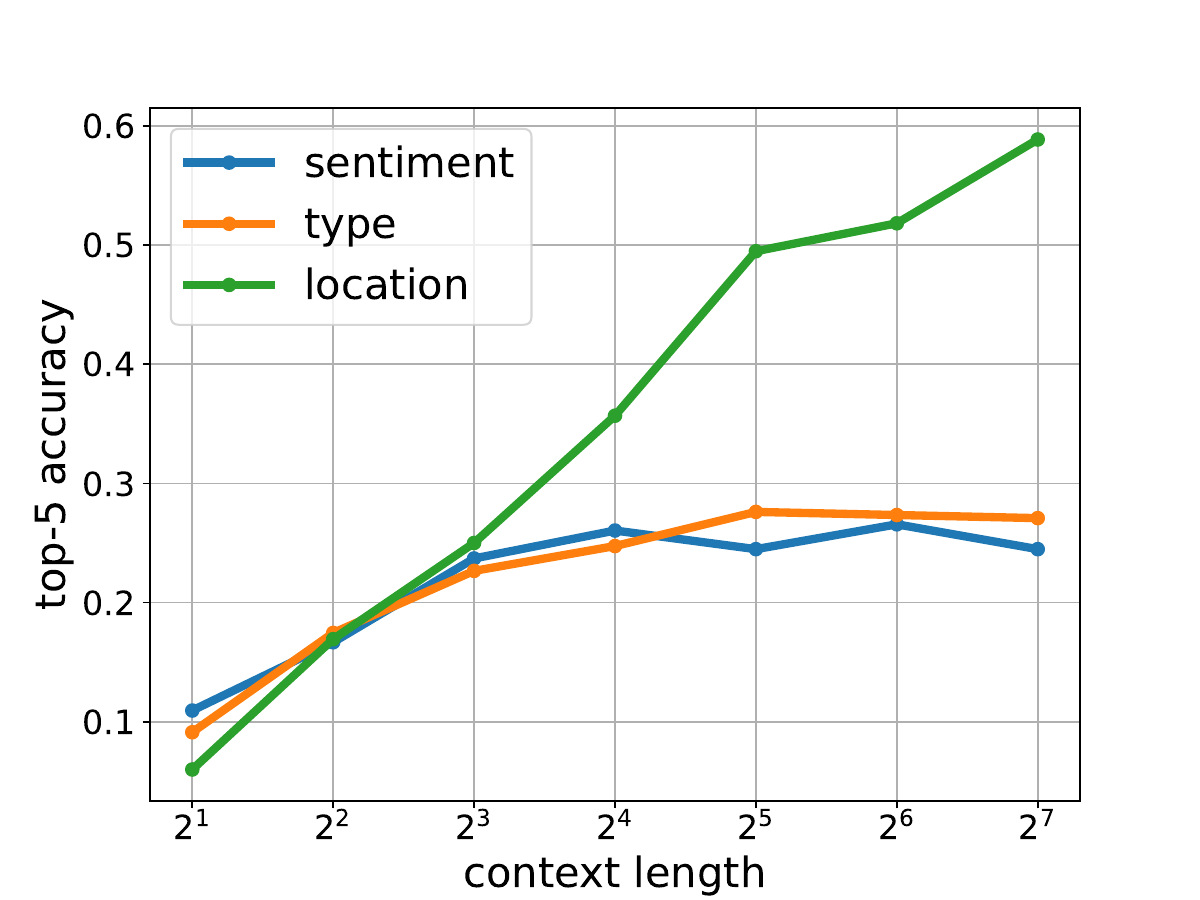}
        \vspace{-5pt}
        \label{fig:realLLM-selection3}
        \caption{Location}
    \end{subfigure}
    \caption{The top-5 accuracy of using (a)``sentiment", (b)``type", or (c)``location" as the classification criterion for in-context examples in a test prompt. The accuracy of using the true underlying criterion to predict is significantly higher than the other two. This suggests that LLMs can perform algorithm selection in natural language tasks.}
    \label{fig:realLLM-selection}
\end{figure}

\begin{takeawaybox}[title=Summary of the Empirical Results \& Connections with the Existing Theories (III)]
    Despite the impressive empirical performance of real-world LLMs in solving seemingly novel tasks through ICL, we observe that when faced with an OOD task, ICL operates by identifying the most suitable pretraining meta-distribution based on test error and input distribution discrepancies, and then attempts to find an optimal solution within that meta-distribution. Notably, this process occurs independently of the downstream test distribution. See Appendix \ref{app:bayesian-interpretation} for the potential connection between such empirical observations and the Bayesian framework work.
    \end{takeawaybox}

\section{Conclusion}

In this work, we empirically find that Transformers struggle to generalize beyond the pretraining function classes when given downstream in-context examples of OOD tasks. Instead, ICL tries to seek a near-optimal solution within the pretraining function classes. We further investigate the widely observed capability of ICL to perform classification. We reveal that it is a composition of ID prediction and retrieval rather than an OOD generalization ability. We also examine ICL’s performance on OOD tasks after pretraining on multiple tasks. Our theoretical and empirical analysis reveals ICL's preference for low-test-error functions, i.e., ICL tends to implement pretraining function classes with low test error in the test context. This finding, alongside previous work \citep{lin2024dual}, highlights two key factors that determine how ICL will implement the prediction function based on the testing context and pretraining tasks: the distance between the training and testing input distributions, and the ability of a pretraining function to solve the test task.

\bibliography{iclr2025_conference}
\bibliographystyle{iclr2025_conference}

\newpage
\appendix

\section{Comparison with Related Works and Additional Discussions}
\subsection{The Capability of ICL to Learn New Tasks}
Besides studies indicating that ICL can learn new weights of linear regression \citep{garg2022can, raventos2024pretraining, zhang2023trained, akyurek2022learning}, other research has found that LLMs can tackle tasks that are unlikely to have been encountered during pretraining. For example, \citet{pan2023context} showed that LLMs perform better than random guessing on classification tasks with meaningless labels. \citet{kossen2024context} demonstrate that ICL can identify authorship based on writing style in private communication messages not included in the pretraining corpus. Additionally, \citet{vacareanu2024words} found that large-scale LLMs can learn various linear and non-linear functions from context. We argue that these findings do not contradict our work. While the LLMs may not have seen exactly the same tasks, there is no guarantee that they haven't encountered tasks from a similar distribution in their pretraining corpus. For instance, the LLMs could have been pretrained on a corpus containing authorship identification tasks or on statistical data encompassing different functions. Our work does not claim that ICL cannot generalize to new task instances; rather, it highlights the limitation in generalizing to an unseen input-label distribution.
Additionally, \citet{yadlowsky2023pretraining} finds that ICL struggles to generalize to testing function classes that are unseen during training (e.g., convex combinations or extreme versions of the pretraining functions). They didn't delve into how ICL behaves on OOD data, while we reveal that it implements the pretraining functions.

\subsection{The Algorithm-selection Mechanism of ICL} Recent works by \citet{bai2024transformers, wang2024context} have uncovered the algorithm selection phenomenon, demonstrating that Transformers pretrained on both linear regression and classification tasks perform well when presented with the context of either task during ICL inference. Theoretically, they show that a Transformer with specific parameters can achieve algorithm selection. \citet{yadlowsky2023pretraining} empirically found that ICL selects the optimal pretraining function class after observing in-context examples from a function class present in the pretraining data mixture. However, the algorithm selection experiments in these studies are limited to scenarios where the test functions are among the training functions. In this work, we empirically and theoretically demonstrate that the algorithm selection phenomenon broadly occurs when given downstream context from arbitrary function classes. To the best of our knowledge, we are the first to reveal the factors that determine the selection process.

\subsection{The Bayesian-optimal Perspective for Understanding ICL}
\label{app:related-work-bayesian}
Many studies have found that ICL makes Bayes-optimal predictions \citep{xie2021explanation, wies2024learnability, zhang2023and, lin2024dual}. However, these works have certain limitations that may reduce their practical applicability in predicting ICL behavior in general scenarios. 1) Limited empirical verification. \cite{wies2024learnability} and \cite{zhang2023and} lack empirical verification of their theory on real deep Transformer models; 2) Limited theoretical settings: in-distribution tasks. \cite{wies2024learnability} assumes the downstream tasks are components of the pretraining distribution; \cite{xie2021explanation} assumes that the latent concept of the test task $\theta^*$ is within the pretraining task set $\Theta$; In \cite{lin2024dual}, the training and testing tasks are all linear regression with weights sampled from Gaussian distribution. 3) Limited implications of the theoretical results: although \cite{xie2021explanation, zhang2023and} prove that ICL can infer a task concept $\theta$ based on the downstream test context $S_{test}$, they don't reveal how $S_{test}$ concretely affects the posterior distribution $P(\theta|S_{test})$ of the latent task concept $\theta$ inferred by the model that determines the downstream ICL prediction, especially when the true downstream task $\theta^*$ is OOD. Our work verifies and extends previous findings to a more general setting by using real deep Transformers and evaluating ICL on OOD tasks that significantly differ from the training tasks. For the first time, we also reveal how the interaction between the downstream distribution and the pretraining distribution affects ICL predictions (see Section \ref{sec:algo-selection}).

In contrast, \cite{raventos2024pretraining} claim that ICL can exhibit non-Bayesian properties. They empirically demonstrate that when given sufficiently diverse pretraining tasks (linear regression vectors), ICL can outperform the Bayesian estimator on a new test distribution. However, the distributional shift in their setup might not be substantial enough to show that ICL can truly adapt to a new downstream distribution, which is considered to be ``non-Bayesian" by \cite{raventos2024pretraining}. In their setting, both the test and training vectors are sampled from the standard Gaussian distribution, and the only source of "distributional shift" comes from the finite size of the training set, which can only partially reflect the test distribution. Our work refutes their findings by showing that when the test distribution is significantly shifted, increasing the number of ID tasks may not help ICL generalize to it.

\subsection{The Bayesian Interpretation for Our Empirical Findings}
\label{app:bayesian-interpretation}
Although current Bayesian theories for ICL are too vacuous to predict the performance of deep Transformers on real OOD tasks (see \ref{app:related-work-bayesian}), the Bayesian framework shows promise as a potential lens for interpreting our empirical findings. Here we provide some intuitive interpretations for the findings in Section \ref{sec:icl-makes-id-predictions}, \ref{sec:abstract-label-learning}, and \ref{sec:algo-selection} from a Bayesian perspective.

Consider the predicted distribution $p_\theta(\boldsymbol{y}_T|\boldsymbol{x}_{1:T})$ given by a pretrained model $\theta$. If we assume that ICL makes Bayesian-like predictions over the test context as \citep{xie2021explanation, wies2024learnability, zhang2023and, lin2024dual} suggested, then the model will first infer a task concept $\phi$ based on the given context $\boldsymbol{x}_{1:T-1}$ and predict $\boldsymbol{y}_T$ using $\phi$ and $\boldsymbol{x}_{1:T}$, i.e.,
\begin{equation}
    p_\theta(\boldsymbol{y}|\boldsymbol{x}_{1:T})=p_\theta(\boldsymbol{y}|\boldsymbol{x}_{1:T},\phi)p_\theta(\phi|\boldsymbol{x}_{1:T-1})
\end{equation}

To explain the results in Section \ref{sec:icl-makes-id-predictions} and Section \ref{sec:algo-selection}: since the true downstream task $\phi^*$ is unseen during pretraining, the inferred posterior distribution $p_\theta(\phi|\boldsymbol{x}_{1:T-1})$ assigns probability mass only to tasks $\phi$ within the pretraining distribution that maximize $p_\theta(\boldsymbol{y}|\boldsymbol{x}_{1:T})$. This accounts for why ICL can only make in-distribution predictions, as shown in Section \ref{sec:icl-makes-id-predictions}, and why ICL prefers pretraining priors with low test error and input distributions similar to those in the test context. Once a task $\phi$ seen during pretraining is identified as best fitting the test context $\boldsymbol{x}_{1:T-1}$, the model refines its predictions based on this context. This refinement corresponds to the factor $p_\theta(\boldsymbol{y}|\boldsymbol{x}_{1:T}, \phi)$, explaining how ICL optimizes predictions within its pretraining distribution.

In Section \ref{sec:abstract-label-learning}, the underlying task concept $\phi$ acts as a similarity metric that allows the model to retrieve examples from the context that align with the query. Training on more abstract labels improves the model's ability to estimate a more accurate $\phi$. When the test task is ID, even with OOD labels, ICL can succeed by leveraging the learned $\phi$ to predict the true label. It accomplishes this by retrieving an example $\boldsymbol{x}_i$ from the context that is similar to the query under the $\phi$ metric. However, when the underlying task $\phi^*$ is OOD, the model fails because the learned similarity metric no longer applies effectively.

\subsection{Discussion of the Setup of Our Theory}
\label{app:discussion-of-our-theory}
Notably, our theoretical result in Section \ref{sec:theory} doesn't assume a Transformer architecture, while previous theoretical works of understanding ICL often adopt Transformers with oversimplified assumptions on their parameters or structures \citep{ahn2024transformers, zhang2023trained, huang2023context, collins2024context}. Additionally, our analysis shows that models pretrained on the ICL tasks can implement algorithm selection during ICL inference following \citet{lin2024dual}. In contrast, prior work on algorithm selection \citep{bai2024transformers} only shows that a specific set of parameters in a simplified ReLU Transformer can enable algorithm selection. However, the parameter construction is complex and somewhat tricky, and there is no theoretical or experimental guarantee that Transformers exhibiting algorithm selection will necessarily implement these parameters.

\section{Additional Experimental Results}
\subsection{Understanding the Effect of Training on More Diverse Retrieval Tasks from the Attention Scores}
\label{app:prefix-matching-score}
To further validate that the retrieval ability is evoked after trained on more random mappings, following \citet{crosbie2024induction}, we construct another retrieval task and visualize the \textit{prefix matching score} of all attention heads of the three pretrained models in Figure \ref{fig:matching-score}. The prefix matching score is calculated by averaging the attention values from each token $t_i$ to the tokens after the same token as $t_i$ in earlier positions in the sequence, which correlates positively with the retrieval performance \citep{singh2024needs}. In Figure \ref{fig:matching-score}, we observe that the model best at the retrieval task in Figure \ref{fig:retrieval} exhibits more heads with high matching scores, further demonstrating it gains the retrieval ability by training on more retrieval sequences.

\begin{figure}[htbp]
    \centering
    \begin{subfigure}{0.35\textwidth}
        \centering
        \includegraphics[width=\textwidth]{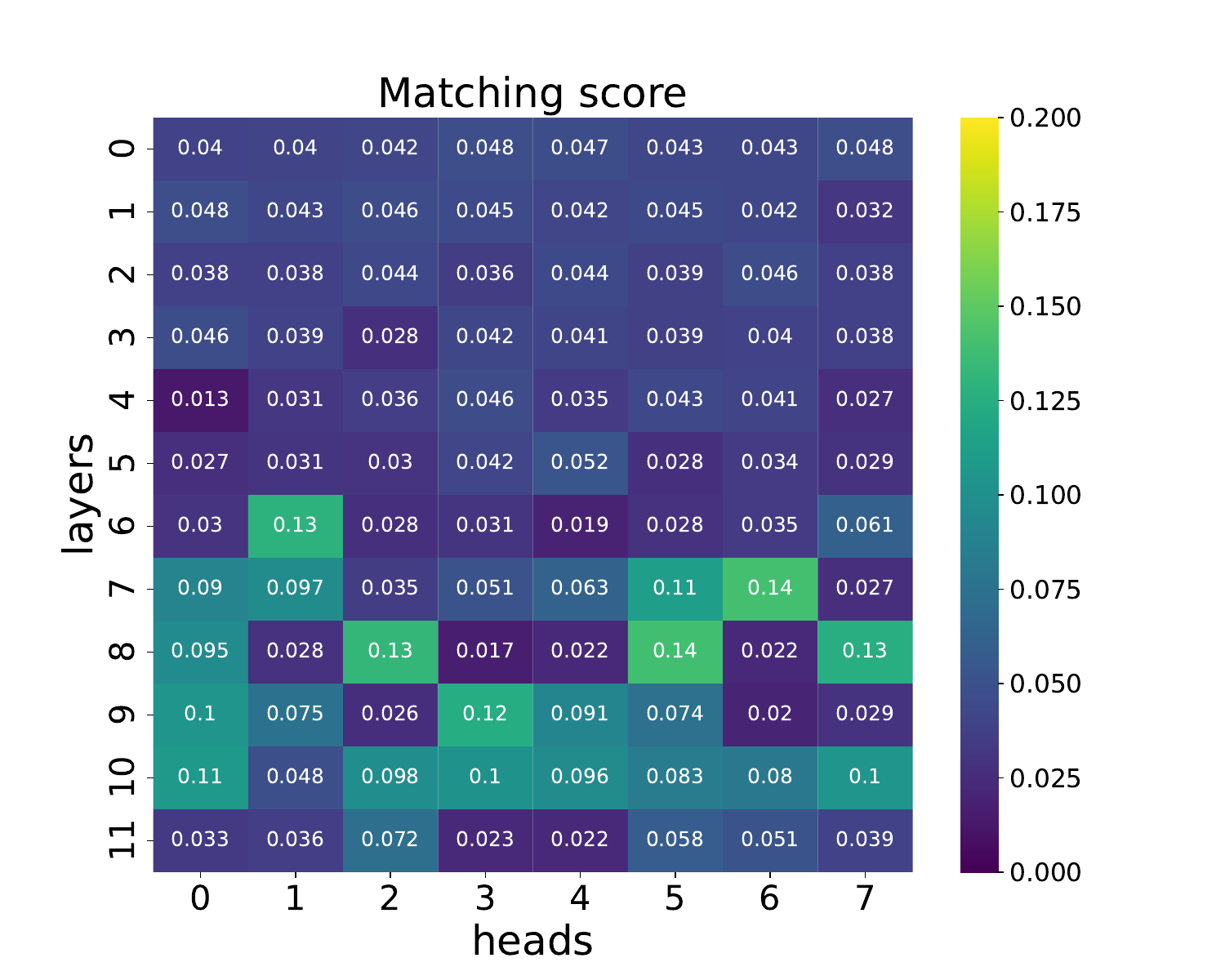}
        \vspace{-5pt}
        \label{fig:induction-head-sub1}
        \caption{PT $s\sim \mathcal{U}(50, 150)$}
    \end{subfigure}
    \hspace{-20pt}
    \begin{subfigure}{0.35\textwidth}
        \centering
        \includegraphics[width=\textwidth]{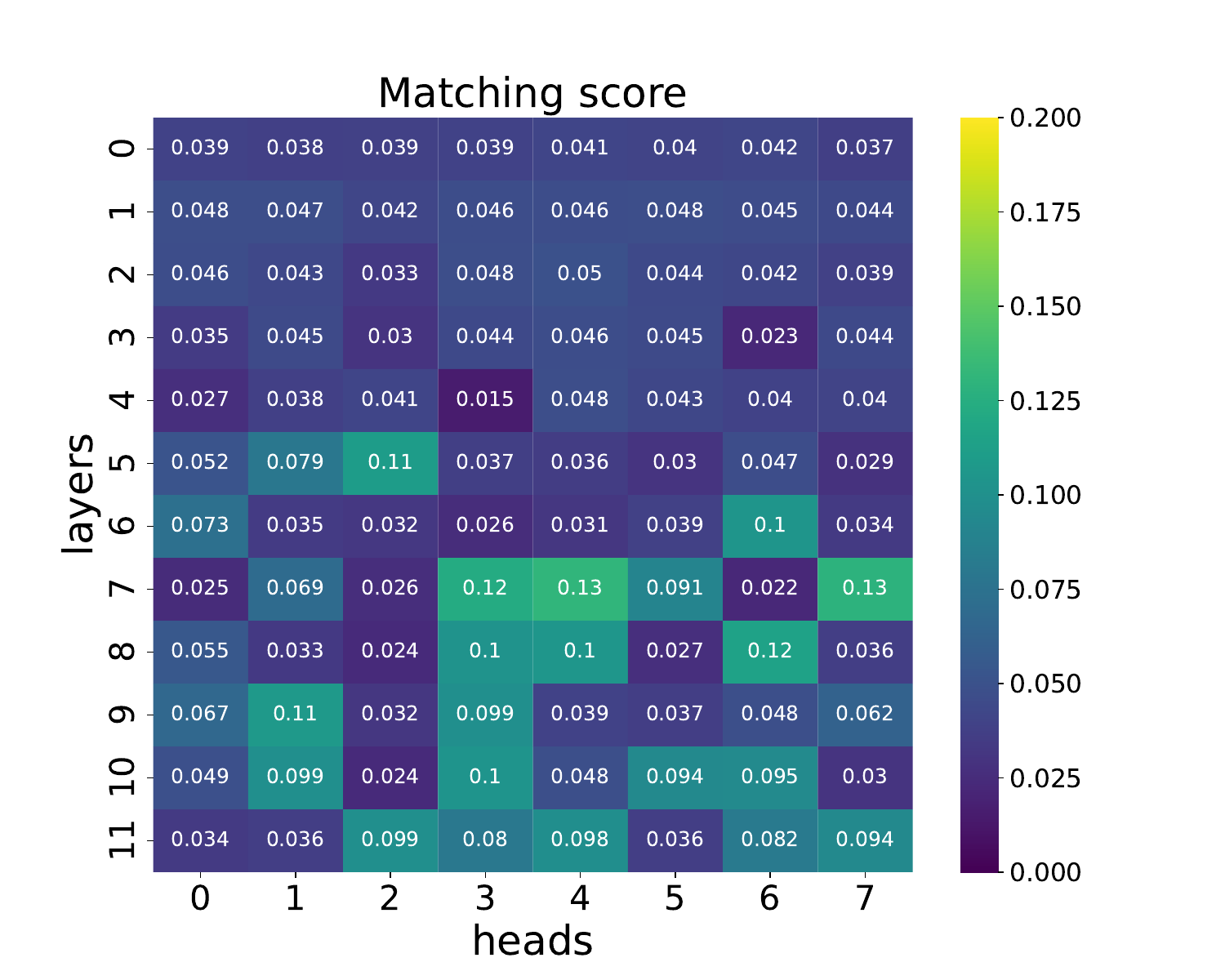}
        \vspace{-5pt}
        \label{fig:induction-head-sub2}
        \caption{PT $s\sim \mathcal{U}(50, 250)$}
    \end{subfigure}
    \hspace{-20pt}
    \begin{subfigure}{0.35\textwidth}
        \centering
        \includegraphics[width=\textwidth]{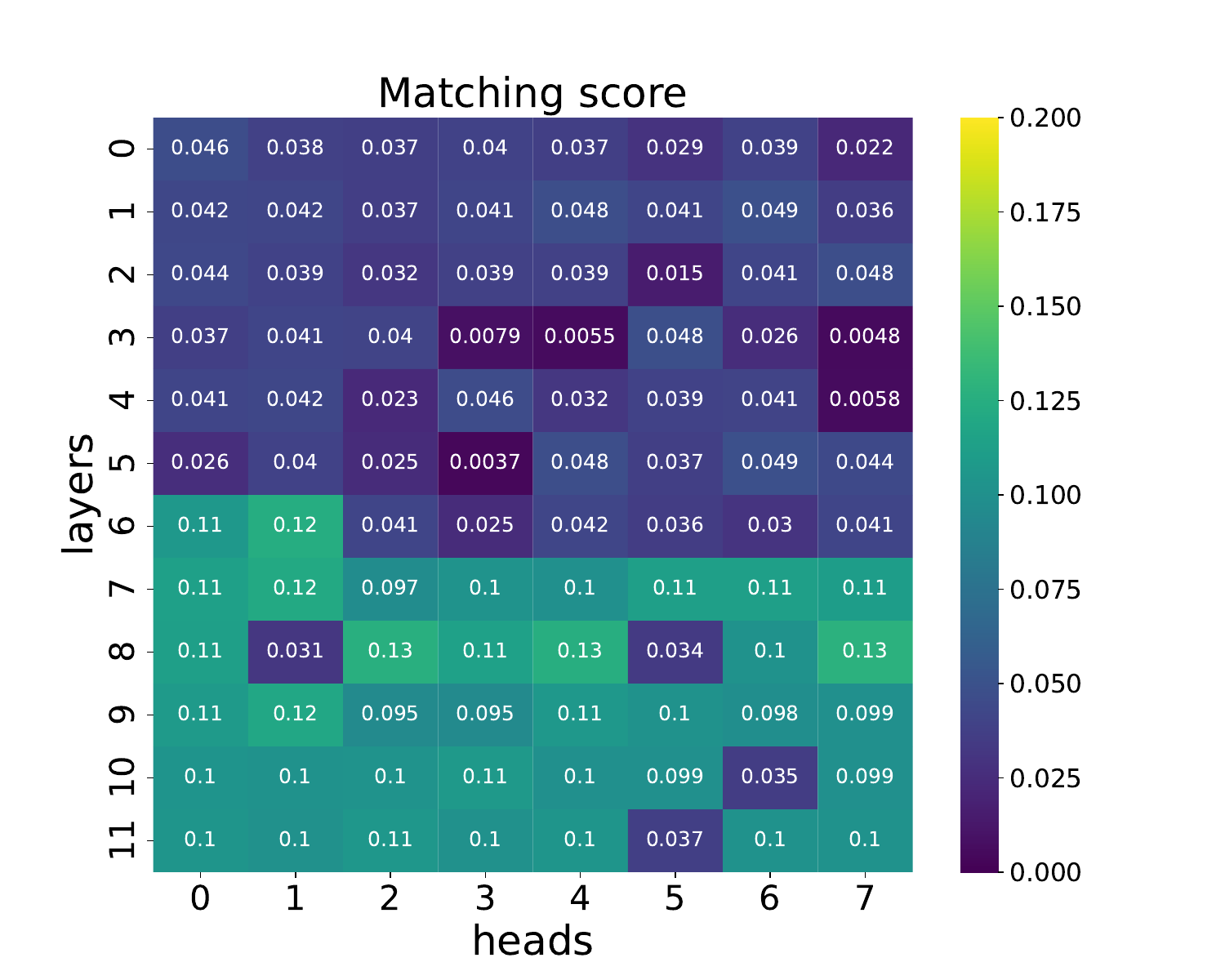}
        \vspace{-5pt}
        \label{fig:induction-head-sub3}
        \caption{PT $s\sim \mathcal{U}(50, 450)$}
    \end{subfigure}
    \caption{The matching score of all attention heads of models trained on the retrieval task. ``PT" denotes ``pretrained on". Each subfigure corresponds to a different pretrained model. The model of (c) exhibits more heads with high matching scores, which is also the most performant model in the retrieval task in Figure \ref{fig:retrieval}.}
    \label{fig:matching-score}
\end{figure}

\subsection{The Synthetic Word Classification is Not That Hard to Solve If It's In Distribution}
\label{app:word-classification-easy}
To show the failure in the synthetic word classification in Section \ref{sec:llama-ood-task} is mainly due to its OOD nature rather than it's intrinsically too hard to learn, we train a GPT-2 to perform the same task as in Section \ref{sec:llama-ood-task}. In this task, the $\boldsymbol{x}_i$ and $\boldsymbol{y}_i$ are generated in the same way as Section \ref{sec:llama-ood-task}. The only modification is that we use a smaller predefined vector embedding $E'\in \mathbb{R}^{10000\times 20}$ ($E_{llama}\in \mathbb{R}^{32000\times 4096}$ in the experiment in Section \ref{sec:llama-ood-task}). The results in Figure \ref{fig:word-classification-gpt2} show that when $\boldsymbol{W}$ has been encountered during pretraining, ICL can well address this task.
\begin{figure}
    \centering
    \includegraphics[width=0.5\linewidth]{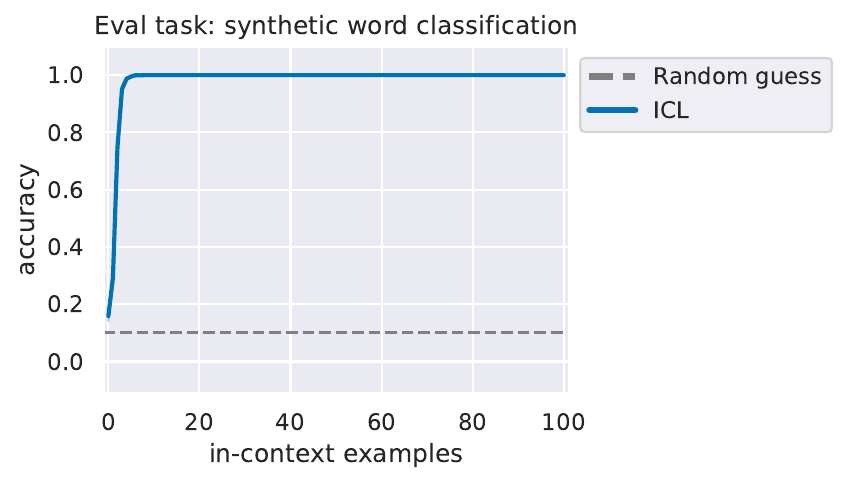}
    \caption{Test error of the GPT-2 trained and evaluated on the same synthetic OOD word classification task as in Section \ref{sec:llama-ood-task}.  }
    \label{fig:word-classification-gpt2}
\end{figure}

\subsection{Evaluating the Synthetic OOD Classification Task on Llama-2-7B}
\label{app:word-classification-llama-2}
We also evaluate Llama-2-7B on the same OOD word classification task and the retrieval task as in Section \ref{sec:llama-ood-task}. Figure \ref{fig:word-classification-llama-2} shows the same observations as in Figure \ref{fig:word-classification-llama-3} that the LLM can well address the retrieval task but fails to learn the OOD function $\boldsymbol{W}$. In this experiment, we set the length of the repeating sequence to be 10. We can observe that the accuracy of retrieval rapidly increases after seeing 10 in-context examples. This demonstrates that learning novel functions from the context is challenging for real-world pretrained LLMs, but the LLMs are good at retrieving.

\begin{figure}[htbp]
    \centering
    \begin{subfigure}{0.4\textwidth}
        \centering
        \includegraphics[width=\textwidth]{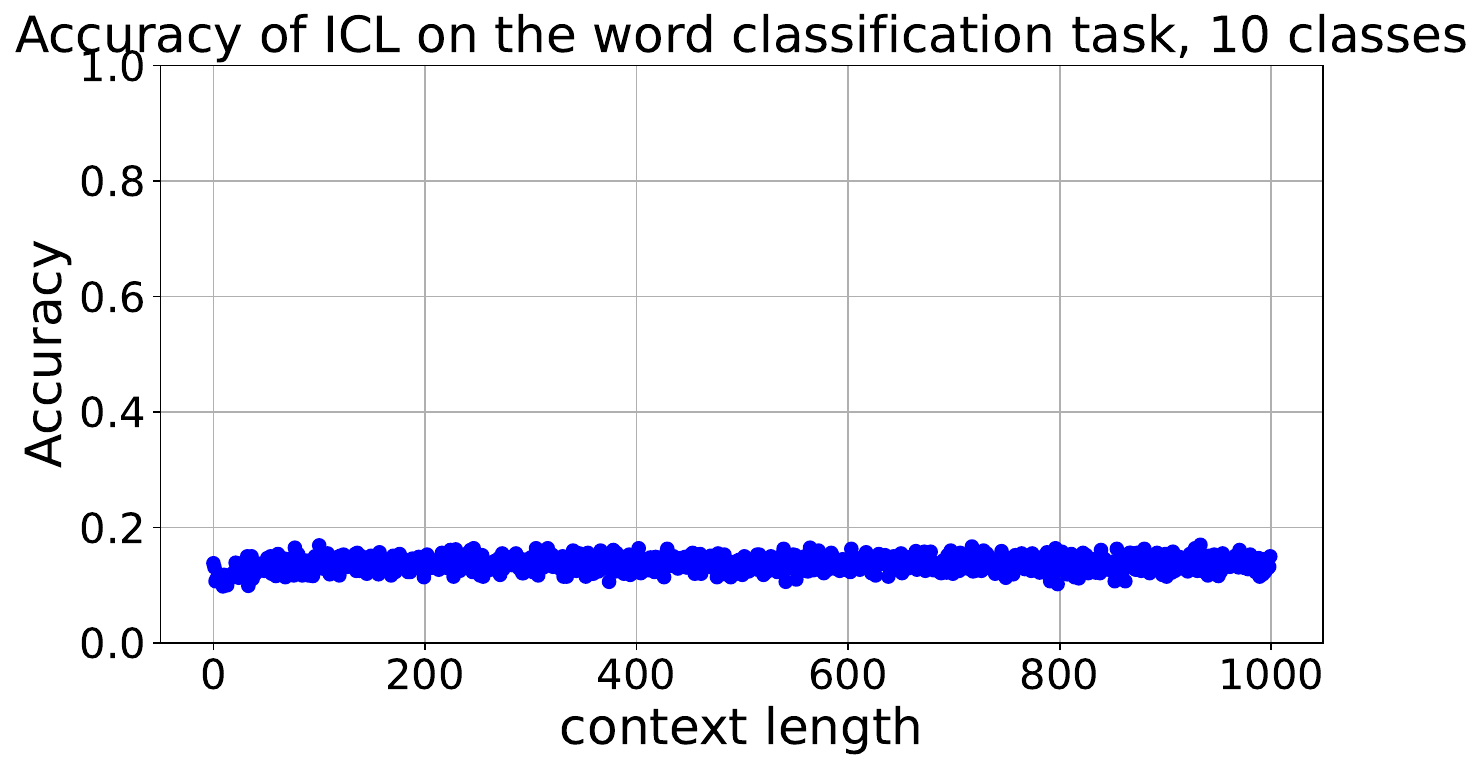}
        \label{fig:llama-word-classification}
        \vspace{-8pt}
        \caption{}
        \vspace{-5pt}
    \end{subfigure}
    \begin{subfigure}{0.4\textwidth}
        \centering
        \includegraphics[width=\textwidth]{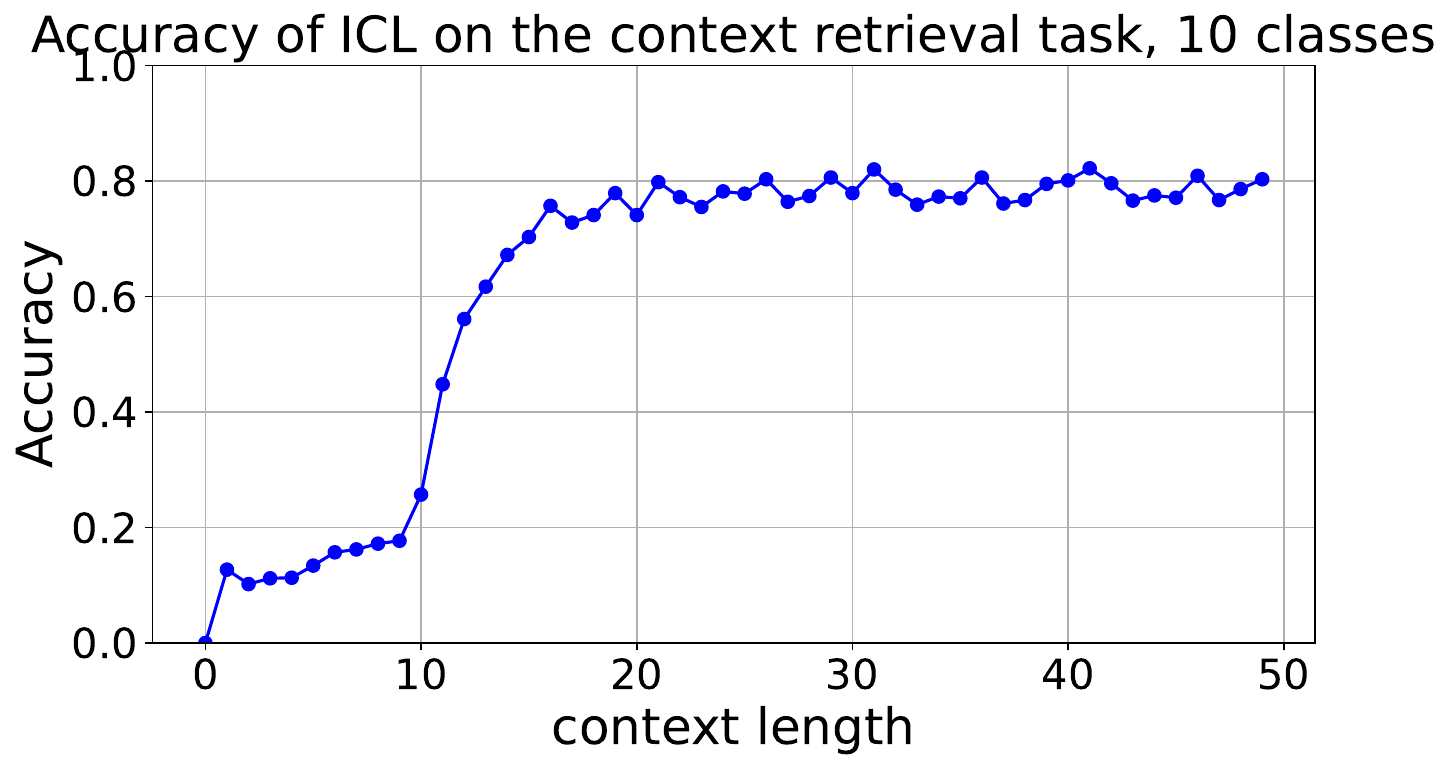}
        \label{fig:llama-word-retrieval}
        \vspace{-8pt}
        \caption{}
        \vspace{-5pt}
    \end{subfigure}
    \caption{The ICL accuracy of Llama-2-7B on the synthetic tasks. (a) the synthetic word classification task. (b) the synthetic word retrieval task.}
    \label{fig:word-classification-llama-2}
\end{figure}

\subsection{Will Generalization Capabilities Emerge from Increasing the Number of Training Tasks?}
\label{sec:task-diversity}
Recent work by \cite{raventos2024pretraining} empirically demonstrates that when both the training and test tasks are linear regression, and the number of training vectors exceeds a certain "task diversity threshold" (approximately $2^{14} \sim 2^{15}$), ICL can generalize from a finite training set sampled biasedly from $\mathcal{N}(0,1)$ to the test distribution $P_{\text{test}} = \mathcal{N}(0,1)$ (see Appendix \ref{app:related-work-bayesian} for details). We investigate whether similar phenomena persist for test tasks with larger distributional shifts. We train models using varying numbers of linear regression vectors and evaluate them on quadratic and ReLU neural network regression tasks. In Figure \ref{fig:task-diversity}, we find that training on a vast number of in-distribution (ID) functions does not yield any improvements, providing further evidence that ICL may struggle to achieve OOD generalization.

\begin{figure}[htbp]
    \centering
    \begin{subfigure}{0.29\textwidth}
        \centering
        \includegraphics[width=\textwidth]{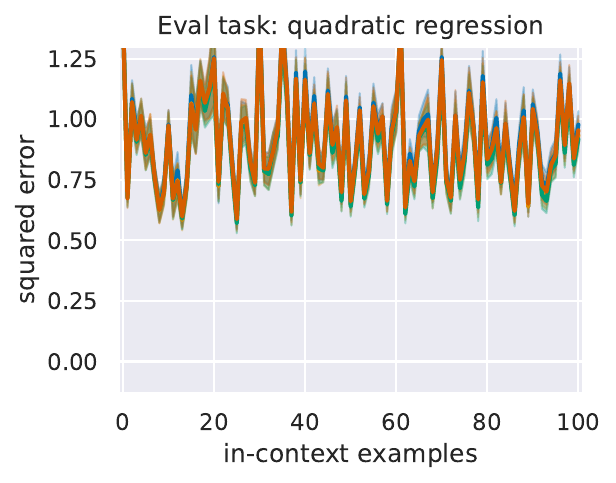}
        \vspace{-6pt}
        \label{fig:task-diversity1}
        \caption{quadratic regression}
    \end{subfigure}
    \begin{subfigure}{0.44\textwidth}
        \centering
        \includegraphics[width=\textwidth]{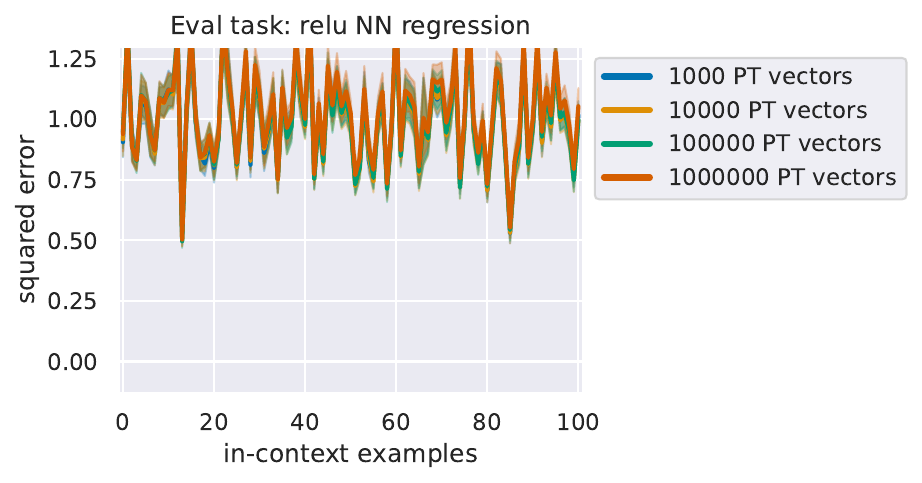}
        \vspace{-6pt}
        \label{fig:task-diversity2}
        \caption{ReLU NN regression}
    \end{subfigure}
    \caption{The ICL test error of models trained on different numbers of linear regression vectors. Even if the number of training vectors (up to $1,000,000\approx 2^{20}$) surpasses the threshold ($2^{14}\sim 2^{15}$) reported by \cite{raventos2024pretraining}, no model exhibits generalization to OOD function classes. }
    \label{fig:task-diversity}
\end{figure}

\section{Experimental Details}
\label{app:exp-detail}
\subsection{Experimental Details in Section \ref{subsec:main} and Section \ref{sec:exp-algo-selection}}
\label{app:exp-detail-1}
\textbf{Definitions of the function classes.} The function classes in Figure \ref{fig:main} and Figure \ref{fig:algo-selection} are: 
\begin{itemize}
    \item Linear regression: $y_i = \boldsymbol{w}^\top \boldsymbol{x}_i$, where $\boldsymbol{w}$, $\boldsymbol{x}_i\in \mathbb{R}^d$ and $\boldsymbol{w}$, $\boldsymbol{x}_i\sim \mathcal{N}(0,1)$.
    \item Quadratic regression: $y_i = \boldsymbol{w}^\top (\boldsymbol{x}_i)^2$, where $\boldsymbol{w}$, $\boldsymbol{x}_i\in \mathbb{R}^d$ and $\boldsymbol{w}$, $\boldsymbol{x}_i\sim \mathcal{N}(0,1)$, $(\boldsymbol{x}_i)^2$ denotes the element-wise square of $\boldsymbol{x}_i$.
    \item 2-layer ReLU network regression: $y_i=\boldsymbol{w}_1^\top \text{ReLU}(\boldsymbol{w}_2 \boldsymbol{x}_i)$, where $\boldsymbol{w}_1\in \mathbb{R}^{d'}$, $\boldsymbol{w}_2\in \mathbb{R}^{d'\times d}$, and $\boldsymbol{x}_i\in \mathbb{R}^d$. $\boldsymbol{w}_1$, $\boldsymbol{w}_2$, $\boldsymbol{x}_i\sim \mathcal{N}(0,1)$.
    \item Square root linear regression:  $y_i = \boldsymbol{w}^\top \sqrt{\boldsymbol{x}_i}$, where $\boldsymbol{w}$, $\boldsymbol{x}_i\in \mathbb{R}^d$ and $\boldsymbol{w}$, $\boldsymbol{x}_i\sim \mathcal{N}(0,1)$, $(\boldsymbol{x}_i)^2$ denotes the element-wise square root of $\boldsymbol{x}_i$.
    \item Cubic linear regression: $y_i = \boldsymbol{w}^\top (\boldsymbol{x}_i)^3$, where $\boldsymbol{w}$, $\boldsymbol{x}_i\in \mathbb{R}^d$ and $\boldsymbol{w}$, $\boldsymbol{x}_i\sim \mathcal{N}(0,1)$, $(\boldsymbol{x}_i)^2$ denotes the element-wise cube of $\boldsymbol{x}_i$.
    \item Linear+quadratic regression: $y_i = \boldsymbol{w}_1^\top (\boldsymbol{x}_i)^2 + \boldsymbol{w}_2^\top \boldsymbol{x}_i$, where $\boldsymbol{w}_1$, $\boldsymbol{w}_2$, $\boldsymbol{x}_i\in \mathbb{R}^d$ and $\boldsymbol{w}_1$, $\boldsymbol{w}_2$, $\boldsymbol{x}_i\sim \mathcal{N}(0,1)$.
    \item 2-layer Sigmoid network: $y_i=\boldsymbol{w}_1^\top \text{Sigmoid}(w_2 \boldsymbol{x}_i)$, where $\boldsymbol{w}_1\in \mathbb{R}^{d'}$, $\boldsymbol{w}_2\in \mathbb{R}^{d'\times d}$, and $\boldsymbol{x}_i\in \mathbb{R}^d$. $\boldsymbol{w}_1$, $\boldsymbol{w}_2$, $\boldsymbol{x}_i\sim \mathcal{N}(0,1)$.
\end{itemize}

\textbf{Baseline models in Figure \ref{fig:main}.} The models of each pretraining hypothesis class are implemented by training a neural network that yields functions of that hypothesis class. For example, a linear regression weight $w$ can be implemented by a single linear layer. The models are optimized using SGD with learning rate 1e-3 for 1000 steps.

\subsection{Experimental Details for Section \ref{sec:realLLM-reverse} }
\label{app:realLLM-reverse}
For the reversed-label experiment, we choose four tasks: Antonym, Capital-country, English-French, and English-German. The original datasets are adopted from \cite{todd2023function}. The top-1 accuracy is computed as follows: compute the top-1 accuracy for each token predicted by the model, based on the token length of the ground-truth label word. For each context length, we compute the average accuracy over 128 test examples.

\subsection{Experimental Details for Section \ref{subsec:abstract-label-learning}}
\label{app:exp-detail-2}
We now provide additional details regarding the experiments of Figure \ref{fig:matching-score}
. Following \citet{crosbie2024induction}, we generated a dataset consisting of 100 sequences of random tokens, each containing repeated sub-sequences. The task is to predict the next token that follows the last token in each sequence. This task can only be completed by retrieving the last token from the context and predicting its subsequent token.

\subsection{Experimental Details for Section \ref{sec:llama-ood-task}}
\label{app:llama-ood-task}
We uniformly sample 1000 word vectors $\boldsymbol{x}_i\in \mathbb{R}^{d}$ from the word embedding $E\in \mathbb{R}^{N\times d}$ of the pretrained Llama-3-8B, where $N=128256$ and $d=4096$. Then we sample a task weight $\boldsymbol{W}\in \mathbb{R}^{d'\times C}$ from standard Gaussian distribution that only takes the first $d'$ dimensions of $\boldsymbol{x}_i$ (denoted as $\boldsymbol{x}_i[:d']$) to compute a probability distribution over $C$ classes: $p_i = \boldsymbol{x}_i[:d']^\top \boldsymbol{W} \in \mathbb{R}^{C}$. Next, we set the label vectors $\boldsymbol{y}_i = E_{\arg \max_j p_i[j] + s }\in \mathbb{R}^{d}$, where $s=10000$ is a offset. We set $d'=30\ll d=4096$ to reduce the complexity of the task. Hence, $\boldsymbol{x}_i$ are classified into $C$ labels words $E[s:s+C]$. The predicted token of $\boldsymbol{x}_i$ is computed as: $ \arg \max_j \hat{p}_i[j],  ~ j\in \{s,s+1,...,s+C-1\} $, where $\hat{p}_i$ is the output of the last linear layer of Llama-3-8B given $\boldsymbol{x}_i$.

\subsection{Experimental Details for Section \ref{sec:realLLM-selection}}
\label{app:realLLM-selection}
In this section, we present some details about the setups for the ambiguous classification task. The label mapping rule is presented in Table \ref{table:realLLM-selection-setting}. For each context length, we compute the average accuracy over 128 test examples. 

\begin{table}[htbp]
    \centering
    \caption{Experiment setting of Figure \ref{fig:algo-selection-mechanism1} and Figure \ref{fig:algo-selection-mechanism2}. ``PT" and ``DS" are short for ``pretraining" and ``downstream", respectively.}
    \resizebox{0.7\linewidth}{!}{
        \begin{tabular}{ccc}
            \toprule
            Classification criterion & Original labels & Labels presented in the context \\
            \midrule
           \multirow{3}{*}{sentiment}& ``positive" & ``RqF"\\
            & ``neutral" & ``IwZ" \\
            & ``negative" & ``SdK" \\
            \midrule
            \multirow{3}{*}{type}& ``science" & ``RqF"\\
            & ``sports" & ``IwZ" \\
            & ``arts" & ``SdK" \\
            \bottomrule
            \multirow{3}{*}{location}& ``Asia" & ``RqF"\\
            & ``Europe" & ``IwZ" \\
            & ``Africa" & ``SdK" \\
            \bottomrule
        \end{tabular}
    }
    \label{table:realLLM-selection-setting}
\end{table}

\textbf{Prompt examples.}  Here we present some in-context examples of the input prompt of using different classification criteria.

\begin{promptbox}[title=Using ``sentiment" as the classification criterion.]
Q: The groundbreaking discovery made by Japanese scientists has revolutionized renewable energy. 

A: RqF gray{\textit{\# Original label: ``positive"}}

\vspace{\baselineskip} 

Q: A chess championship occurred in Russia, featuring players from around the continent.

A: IwZ gray{\textit{\# Original label: ``neutral"}}

\end{promptbox}

\begin{promptbox}[title=Using ``type" as the classification criterion.]
Q: A regional basketball league was formed in Kenya to promote the sport locally.

A: IwZ gray{\textit{\# Original label: ``sports"}}

\vspace{\baselineskip} 

Q: The breathtaking architectural exhibition in Dubai left visitors absolutely awestruck.

A: SdK gray{\textit{\# Original label: ``arts"}}

\end{promptbox}

\begin{promptbox}[title=Using ``location" as the classification criterion.]
Q: A scientific paper from Finland explores new methodologies in data analysis.

A: IwZ gray{\textit{\# Original label: ``Europe"}}

\vspace{\baselineskip} 

Q: An astronomy workshop was conducted in Ethiopia for students interested in space.

A: SdK gray{\textit{\# Original label: ``Africa"}}

\end{promptbox}

\textbf{Accuracy computation.} For a given label, the method to calculate top-5 accuracy is as follows: compute the top-5 accuracy for each token predicted by the model, based on the token length of the ground-truth label word. For a classification criterion other than the one selected in the current sequence, to verify whether the model's prediction distribution across all test samples approaches the label distribution under that criterion, we select the permutation among all possible mappings between original labels and meaningless strings that yields the highest model prediction accuracy to compute the accuracy.

\section{Existing Theoretical Evidence Supporting that ICL Makes ID Predictions}

One recent work \citep{zhang2023trained} theoretically proved that a one-layer linear self-attention model (LSA, defined in Appendix \ref{app:theo-detail-imple-LSA}) pretrained on a linear regression task will still implement the linear predictor given downstream in-context examples of arbitrary new function classes, under some assumptions on the initialization of the Transformer weight matrices. We restate the Theorem 4.2 of \citet{zhang2023trained} as Lemma \ref{lem:imple-LSA} below:
\begin{lemma}
\label{lem:imple-LSA} 
    (Theorem 4.2 of \citet{zhang2023trained}, informal) Let $\mathcal{D}$ be a distribution over $(\boldsymbol{x}, y) \in \mathbb{R}^d \times \mathbb{R}$, whose marginal distribution on $x$ is $\mathcal{D}_x=$ $\mathcal{N}(0, \Lambda)$. Assume the test prompt is of the form $P=\left(\boldsymbol{x}_1, y_1, \ldots, \boldsymbol{x}_T, y_T, \boldsymbol{x}_{\text {query }}\right)$, where $\left(\boldsymbol{x}_i, y_i\right),\left(\boldsymbol{x}_{\text {query }}, y_{\text {query }}\right) \stackrel{\text { i.i.d. }}{\sim} \mathcal{D}$. The prediction risk on the test query $y_{query}$ of an arbitrary task satisfies:
$$
\mathbb{E}\left(\widehat{y}_{\text {query }}-y_{\text {query }}\right)^2  =\underbrace{\min _{\boldsymbol{w} \in \mathbb{R}^d} \mathbb{E}\left(\left\langle \boldsymbol{w}, \boldsymbol{x}_{\text {query }}\right\rangle-y_{\text {query }}\right)^2}_{\text {Error of best linear predictor }}  + const,
$$
where $const$ is a constant depending on the downstream context, and the expectation is over $\left(\boldsymbol{x}_i, y_i\right),\left(\boldsymbol{x}_{\text {query }}, y_{\text {query }}\right) \stackrel{\text { i.i.d. }}{\sim} \mathcal{D}$.
\end{lemma}
Lemma \ref{lem:imple-LSA} serves as a shred of theoretical evidence that ICL can just implement the pretraining function class, while the role of the context examples is to optimize the model within the pretraining hypothesis space.

\label{app:theo-detail-imple-LSA}
Below, we provide the necessary details of the theoretical setting of \citet{zhang2023trained}.

The linear self-attention (LSA) model considered in the Theorem 4.2 of \cite{zhang2023trained} (Lemma \ref{lem:imple-LSA}) is defined as follows:
\begin{equation}
    \label{eq:LSA}
    f_{\mathrm{LSA}}(E ; \theta)=E+W^{P V} E \cdot \frac{E^{\top} W^{K Q} E}{N},
\end{equation}
where $E$ is the input embedding defined as follows:
\begin{equation}
    E=E(P)=\left(\begin{array}{ccccc}
\boldsymbol{x}_1 & \boldsymbol{x}_2 & \cdots & \boldsymbol{x}_N & \boldsymbol{x}_{\text {query }} \\
y_1 & y_2 & \cdots & y_N & 0
\end{array}\right) \in \mathbb{R}^{(d+1) \times(N+1)}.
\end{equation} $W^{PV}$ is obtained by merging the projection and value matrices into a single matrix, and $W^{KQ}$ is attained by merging the query and key matrices into a single matrix. $N$ is the context length.

Now we present the assumption on the attention weights of the linear-attention model in Lemma \ref{lem:imple-LSA}.
\begin{assumption}
\label{ass:imple-LSA}
    (Assumption 3.3 in \cite{zhang2023trained}, initialization). Let $\sigma>0$ be a parameter, and let $\Theta \in \mathbb{R}^{d \times d}$ be any matrix satisfying $\left\|\Theta \Theta^{\top}\right\|_F=1$ and $\Theta \Lambda \neq 0_{d \times d}$. We assume

$$
W^{P V}(0)=\sigma\left(\begin{array}{cc}
0_{d \times d} & 0_d \\
0_d^{\top} & 1
\end{array}\right), \quad W^{K Q}(0)=\sigma\left(\begin{array}{cc}
\Theta \Theta^{\top} & 0_d \\
0_d^{\top} & 0
\end{array}\right)
$$

\end{assumption}

The training objective is to minimize the population risk of the linear regression task:
\begin{equation}
    \label{eq:LSA-training-objective}
    L(\theta)=\lim _{B \rightarrow \infty} \widehat{L}(\theta)=\frac{1}{2} \mathbb{E}_{w_\tau, \boldsymbol{x}_{\tau, 1}, \cdots, \boldsymbol{x}_{\tau, N}, \boldsymbol{x}_{\tau, \text { query }}}\left[\left(\widehat{y}_{\tau, \text { query }}-\left\langle w_\tau, \boldsymbol{x}_{\tau, \text { query }}\right\rangle\right)^2\right],
\end{equation}
where $w_\tau\sim \mathcal{N}(0, I_d)$, $\boldsymbol{x}_{\tau,i}$, $\boldsymbol{x}_{\tau,query}\sim \mathcal{N}(0, \Lambda)$, $\widehat{y}_{\tau, \text { query }}$ is the prediction of the LSA model.

\section{The Lemmas, Assumption, and Proof for Theorem \ref{theo:algo-selection}}
\label{proof:algo-selection}

In this section, we will first present the lemmas and assumption that Theorem \ref{theo:algo-selection} depends on and the provide its proof.

\subsection{Lemmas for Theorem \ref{theo:algo-selection}}
The lemma below states that the closed-form prediction of the model trained on the pretraining data under Assumption \ref{ass:mixed-gaussian}, given the testing context, remains a Gaussian mixture of the reweighted pretraining task weights: 
\begin{lemma}
    \label{lem:mixed-gaussian-pred}
    (Corollary 2 of \citet{lin2024dual}, closed-form ICL prediction of the pretrained model) Denote the model $M^*$ that minimizes the risk on the pretraining data of Assumption \ref{ass:mixed-gaussian}, i.e., $M^* \in \arg \min \frac{1}{T} \sum_{i=0}^{T-1} \underset{S_i\sim P_{tr}}{\mathbb{E}}\left[\left\|M\left(\mathcal{S}_i \oplus \boldsymbol{x}_{i+1}\right)-y_{i+1}\right\|^2\right]$, then the prediction on any sequence $\mathcal{S}_i \oplus \boldsymbol{x}_{i+1}$ by $M^*$ is as follows: $M^*\left(\mathcal{S}_i \oplus \boldsymbol{x}_{i+1}\right)=\left\langle \boldsymbol{x}_{i+1}, \sum_{m=1}^M \tilde{\pi}_m \tilde{\boldsymbol{w}}_m\right\rangle$.
     where $\tilde{\pi}_m$, and $\tilde{\boldsymbol{w}}_m$ depending on both the pretraining task and the downstream context example are given in Lemma 1 of  \citet{lin2024dual}.
\end{lemma}

\subsection{The Assumption for Theorem \ref{theo:algo-selection}}
The assumption below impose some mild requirements on the distribution of the downstream context:
\begin{assumption}
    \label{ass:theo:algo-selection-context}
    (Assumption on the distribution of the downstream context examples.) Assume that: the minimum eigenvalue of the covariance matrix of any in-context example $\boldsymbol{x}_i$ satisfies $\lambda_{\text{min}} ( \boldsymbol{x}_i \boldsymbol{x}_i^\top ) \geq 1$;
    $(\boldsymbol{I}+T\delta_w\boldsymbol{I} )(\boldsymbol{I}+\delta_w \sum_{i=1}^T \boldsymbol{x}_i \boldsymbol{x}_i^\top)^{-1} = \boldsymbol{I}$;
        $\frac{1}{T} \sum_{i=1}^T  2(\boldsymbol{w}_\alpha - \boldsymbol{w}_\beta)^\top \boldsymbol{x}_iy_i \frac{1}{T} \sum_{j=1}^T \left(\boldsymbol{x}_j^\top \boldsymbol{x}_i \frac{y_j}{y_i} - \boldsymbol{x}_i^\top \boldsymbol{x}_i \right) \geq 0$
\end{assumption}

\subsection{Proof for Theorem \ref{theo:algo-selection}}
Now we restate Theorem \ref{theo:algo-selection} as the Theorem \ref{theo:algo-selection-app} below

\begin{theorem}
    \label{theo:algo-selection-app}
     (ICL prediction favors the pretraining function with low error on the context) Given the context $S_k$, if the empirical risk of implementing a function of the pretraining task $\alpha$ is less than that of $\beta$, i.e., $\frac{1}{T} \sum_{t=1}^{T} |\boldsymbol{w}_\beta \boldsymbol{x}_i - y_i|^2 - |\boldsymbol{w}_\alpha \boldsymbol{x}_i - y_i|^2 \geq 0$, 
    then, under some mild Assumptions \ref{ass:theo:algo-selection-context}, we have $\Psi_{\boldsymbol{w}}(\alpha, \beta) \geq 0$. 
    
    Combining Lemma \ref{lem:dual-x-shift}, if the downstream inputs $\boldsymbol{x}_i$, $\boldsymbol{x}_i\sim \mathcal{N}(\boldsymbol{\mu}^*, \tau_x^2 \boldsymbol{I})$ and $\|\boldsymbol{\mu}_\beta - \boldsymbol{\mu}^*\|^2 - \|\boldsymbol{\mu}_\alpha - \boldsymbol{\mu}^*\|^2 \geq 0$ hold, then as $T\rightarrow \infty$, we have $\tilde{\pi}_\alpha/\tilde{\pi}_\beta \geq \pi_\alpha/\pi_\beta$.
\end{theorem}

\begin{proof}
    According to Lemma 1 of \citet{lin2024dual}, 
    \begin{equation}
        r(\alpha, \beta)=\frac{\tilde{\pi}_\alpha}{\tilde{\pi}_\beta}=\frac{\pi_\alpha C_0 c_\alpha^\mu c_\alpha^w}{\pi_\beta C_0 c_\beta^\mu c_\beta^w}=\frac{\pi_\alpha}{\pi_\beta} \exp \left( \Psi_\mu(\alpha, \beta) + \Psi_{\boldsymbol{w}}(\alpha, \beta)\right).
    \end{equation}
    In the Appendix H.1 of \cite{lin2024dual}, they have proved that when the context length $T\rightarrow \infty$, under the first condition in Assumption \ref{ass:theo:algo-selection-context}, $ \lim_{T\rightarrow \infty} \Psi_\mu(\alpha, \beta) = \geq 0$.

    Now we prove that when the empirical risk on the in-context example of pretraining task function $\alpha$ is no more than that of $\beta$, the second term $\Psi_{\boldsymbol{w}}(\alpha, \beta)\geq 0$.

    \begin{equation}
    \begin{aligned}
        &\Psi_{\boldsymbol{w}}(\alpha, \beta)\\
        & =\log \left(\frac{\exp \left(-\frac{\left\|\boldsymbol{w}_\alpha\right\|^2-\left\|\boldsymbol{w}_\alpha+T \delta_w \overline{\boldsymbol{w}}\right\|_{\left(\boldsymbol{I}+T \delta_w \boldsymbol{\Sigma}_{\boldsymbol{w}}\right)^{-1}}^2}{2 \sigma_w^2}\right)}{\exp \left(-\frac{\left\|\boldsymbol{w}_\beta\right\|^2-\left\|\boldsymbol{w}_\beta+T \delta_w \boldsymbol{w}\right\|_{\left(\boldsymbol{I}+T \delta_w \bar{\Sigma}_{\boldsymbol{w}}\right)^{-1}}^2}{2 \sigma_w^2}\right)}\right) \\
        & =\frac{\left\|\boldsymbol{w}_\beta\right\|^2-\left\|\boldsymbol{w}_\beta+T \delta_w \overline{\boldsymbol{w}}\right\|_{\left(\boldsymbol{I}+T \delta_w \overline{\boldsymbol{\Sigma}}_{\boldsymbol{w}}\right)^{-1}}^2}{2 \sigma_w^2}-\frac{\left\|\boldsymbol{w}_\alpha\right\|^2-\left\|\boldsymbol{w}_\alpha+T \delta_w \overline{\boldsymbol{w}}\right\|_{\left(\boldsymbol{I}+T \delta_w \overline{\boldsymbol{\Sigma}}_{\boldsymbol{w}}\right)^{-1}}^2}{2 \sigma_w^2}\\
        &= \frac{ \|\boldsymbol{w}_\beta\|^2 - \left\|\boldsymbol{w}_\beta+ \delta_w \sum_{i=1}^T \boldsymbol{x}_i y_i \right\|_{\left(\boldsymbol{I}+T \delta_w \overline{\boldsymbol{\Sigma}}_{\boldsymbol{w}}\right)^{-1}}^2  }{2 \sigma_w^2} 
        -
        \frac{ \|\boldsymbol{w}_\alpha\|^2 - \left\|\boldsymbol{w}_\alpha+ \delta_w \sum_{i=1}^T \boldsymbol{x}_i y_i \right\|_{\left(\boldsymbol{I}+T \delta_w \overline{\boldsymbol{\Sigma}}_{\boldsymbol{w}}\right)^{-1}}^2  }{2 \sigma_w^2}\\
        &=\frac{ \|\boldsymbol{w}_\beta\|^2 - 
        \left\|(\boldsymbol{w}_\beta - \frac{\sum_{i=1}^T \boldsymbol{x}_i y_i}{T})+ (\boldsymbol{I}+T\boldsymbol{I}\delta_w)  \frac{\sum_{i=1}^T \boldsymbol{x}_i y_i}{T} \right\|_{\left(\boldsymbol{I}+T \delta_w \overline{\boldsymbol{\Sigma}}_{\boldsymbol{w}}\right)^{-1}}^2  }{2 \sigma_w^2} 
        \\
        &-
        \frac{ \|\boldsymbol{w}_\alpha\|^2 - 
        \left\|(\boldsymbol{w}_\alpha - \frac{\sum_{i=1}^T \boldsymbol{x}_i y_i}{T})+ (\boldsymbol{I}+T\boldsymbol{I}\delta_w)  \frac{\sum_{i=1}^T \boldsymbol{x}_i y_i}{T} \right\|_{\left(\boldsymbol{I}+T \delta_w \overline{\boldsymbol{\Sigma}}_{\boldsymbol{w}}\right)^{-1}}^2  }{2 \sigma_w^2} \\
        &\overset{(a)}{=}\| \boldsymbol{w}_\beta - \frac{\sum_{i=1}^T \boldsymbol{x}_i y_i}{T} \|^2_{\boldsymbol{I} - \left(\boldsymbol{I}+T \delta_w \overline{\boldsymbol{\Sigma}}_{\boldsymbol{w}}\right)^{-1}}
        -
        \| \boldsymbol{w}_\alpha - \frac{\sum_{i=1}^T \boldsymbol{x}_i y_i}{T} \|^2_{\boldsymbol{I} - \left(\boldsymbol{I}+T \delta_w \overline{\boldsymbol{\Sigma}}_{\boldsymbol{w}}\right)^{-1}}\\
        &\overset{(b)}{=}\| \boldsymbol{w}_\beta - \frac{\sum_{i=1}^T \boldsymbol{x}_i y_i}{T} \|^2_{\frac{ \delta_w \sum_{i=1}^T \boldsymbol{x}_i \boldsymbol{x}_i^\top}{1+ \delta_w \sum_{i=1}^{T} \boldsymbol{x}_i^\top \boldsymbol{x}_i}}
        -
        \| \boldsymbol{w}_\alpha - \frac{\sum_{i=1}^T \boldsymbol{x}_i y_i}{T} \|^2_{\frac{ \delta_w \sum_{i=1}^T \boldsymbol{x}_i \boldsymbol{x}_i^\top}{1+ \delta_w \sum_{i=1}^{T} \boldsymbol{x}_i^\top \boldsymbol{x}_i}}
    \end{aligned}
    \end{equation}
where equation $(a)$ is due to the third condition in Assumption \ref{ass:theo:algo-selection-context}, equation $(b)$ is by applying the Sherman–Morrison formula. Since $\frac{\delta_w }{1 + \delta_w \sum_{i=1}^{T}}\geq 0$, to prove that $\Psi_{\boldsymbol{w}}(\alpha, \beta) \geq 0$, we only need to show that 
\begin{equation}
\label{proof:ineq-1}
    \| \boldsymbol{w}_\beta - \frac{\sum_{i=1}^T \boldsymbol{x}_i y_i}{T} \|^2_{\sum_{i=1}^T \boldsymbol{x}_i \boldsymbol{x}_i^\top}
        -
        \| \boldsymbol{w}_\alpha - \frac{\sum_{i=1}^T \boldsymbol{x}_i y_i}{T} \|^2_{\sum_{i=1}^T \boldsymbol{x}_i \boldsymbol{x}_i^\top} \geq 0.
\end{equation}

Further, we can derive that the term $\frac{1}{T} \sum_{i=1}^T\|\boldsymbol{w}_\beta - \boldsymbol{x}_i  y_i\|^2_{\boldsymbol{x}_i \boldsymbol{x}_i^T}
        -
        \|\boldsymbol{w}_\alpha - \boldsymbol{x}_i  y_i\|^2_{\boldsymbol{x}_i \boldsymbol{x}_i^T}$ below is non-negative by using the condition 2 in Assumption \ref{ass:theo:algo-selection-context}:
\begin{equation}
    \label{proof:ineq-2}
    \begin{aligned}
        &\frac{1}{T} \sum_{i=1}^T\|\boldsymbol{w}_\beta - \boldsymbol{x}_i  y_i\|^2_{\boldsymbol{x}_i \boldsymbol{x}_i^T}
        -
        \|\boldsymbol{w}_\alpha - \boldsymbol{x}_i  y_i\|^2_{\boldsymbol{x}_i \boldsymbol{x}_i^T}
        \\
        &=\frac{1}{T} \sum_{i=1}^T(\boldsymbol{w}_\beta - \boldsymbol{x}_i  y_i)^\top \boldsymbol{x}_i \boldsymbol{x}_i^T (\boldsymbol{w}_\beta - \boldsymbol{x}_i  y_i)
        -
        (\boldsymbol{w}_\alpha - \boldsymbol{x}_i  y_i)^\top \boldsymbol{x}_i \boldsymbol{x}_i^T (\boldsymbol{w}_\alpha - \boldsymbol{x}_i  y_i)
        \\
        &=\frac{1}{T} \sum_{i=1}^T(\boldsymbol{w}_\beta + \boldsymbol{w}_\alpha - 2\boldsymbol{x}_i  y_i)^\top \boldsymbol{x}_i \boldsymbol{x}_i^T (\boldsymbol{w}_\beta - \boldsymbol{w}_\alpha)
        \\
        &\underset{(c)}{\geq} \frac{1}{T} \sum_{i=1}^T(\boldsymbol{w}_\beta + \boldsymbol{w}_\alpha - 2\boldsymbol{x}_i  y_i)^\top  (\boldsymbol{w}_\beta - \boldsymbol{w}_\alpha)\\
        &= \frac{1}{T} \sum_{i=1}^T \|\boldsymbol{w}_\beta^\top \boldsymbol{x}_i - y_i\|^2 
        -
         \|\boldsymbol{w}_\alpha^\top \boldsymbol{x}_i - y_i\|^2 \underset{(d)}{\geq 0}\\
    \end{aligned}
\end{equation}

where the inequality $(c)$ holds since according to the condition 2 in Assumption \ref{ass:theo:algo-selection-context}, $\boldsymbol{x}_i \boldsymbol{x}_i^T-\boldsymbol{I}$ is positive semi-definite, and the inequality $(d)$ holds since the population downstream risk of $\alpha$ is lower than $\beta$. Therefore, to prove inequality (\ref{proof:ineq-1}), we just need to prove that the l.h.s. of inequality (\ref{proof:ineq-1}) multiplying $\frac{1}{T}$ is not less than $\frac{1}{T} \sum_{i=1}^T\|\boldsymbol{w}_\beta - \boldsymbol{x}_i  y_i\|^2_{\boldsymbol{x}_i \boldsymbol{x}_i^T}$ in Equation (\ref{proof:ineq-2}):
\begin{equation}
\label{proof:ineq-3}
    \frac{1}{T}\left(\| \boldsymbol{w}_\beta - \frac{\sum_{i=1}^T \boldsymbol{x}_i y_i}{T} \|^2_{\sum_{i=1}^T \boldsymbol{x}_i \boldsymbol{x}_i^\top}
        -
        \| \boldsymbol{w}_\alpha - \frac{\sum_{i=1}^T \boldsymbol{x}_i y_i}{T} \|^2_{\sum_{i=1}^T \boldsymbol{x}_i \boldsymbol{x}_i^\top} \right)
        \geq 
        \frac{1}{T} \sum_{i=1}^T\|\boldsymbol{w}_\beta - \boldsymbol{x}_i  y_i\|^2_{\boldsymbol{x}_i \boldsymbol{x}_i^T}
        -
        \|\boldsymbol{w}_\alpha - \boldsymbol{x}_i  y_i\|^2_{\boldsymbol{x}_i \boldsymbol{x}_i^T}.
\end{equation}

First, let's simplify the l.h.s of inequality (\ref{proof:ineq-3}):
\begin{equation}
    \begin{aligned}
    \label{proof:eq-1}
            &\frac{1}{T}\left(\| \boldsymbol{w}_\beta - \frac{\sum_{i=1}^T \boldsymbol{x}_i y_i}{T} \|^2_{\sum_{i=1}^T \boldsymbol{x}_i \boldsymbol{x}_i^\top}
        -
        \| \boldsymbol{w}_\alpha - \frac{\sum_{i=1}^T \boldsymbol{x}_i y_i}{T} \|^2_{\sum_{i=1}^T \boldsymbol{x}_i \boldsymbol{x}_i^\top} \right)\\
        & = \frac{1}{T} \sum_{i=1}^T (\boldsymbol{w}_\beta - \frac{\sum_{j=1}^T \boldsymbol{x}_j y_j}{T})^\top \boldsymbol{x}_i \boldsymbol{x}_i^\top (\boldsymbol{w}_\beta - \frac{\sum_{j=1}^T \boldsymbol{x}_j y_j}{T})
        -
         (\boldsymbol{w}_\alpha - \frac{\sum_{j=1}^T \boldsymbol{x}_j y_j}{T})^\top \boldsymbol{x}_i \boldsymbol{x}_i^\top (\boldsymbol{w}_\alpha - \frac{\sum_{j=1}^T \boldsymbol{x}_j y_j}{T})\\
        &= \frac{1}{T} \sum_{i=1}^T \| \boldsymbol{w}_\beta^\top \boldsymbol{x}_i - \frac{1}{T} \sum_{j=1}^T \boldsymbol{x}_j^\top \boldsymbol{x}_i y_j \|^2
        -
        \| \boldsymbol{w}_\alpha^\top \boldsymbol{x}_i - \frac{1}{T} \sum_{j=1}^T \boldsymbol{x}_j^\top \boldsymbol{x}_i y_j \|^2\\
        & =  \frac{1}{T} \sum_{i=1}^T (\boldsymbol{w}_\beta^\top \boldsymbol{x}_i)^2 - (\boldsymbol{w}_\alpha^\top \boldsymbol{x}_i)^2 + 2(\boldsymbol{w}_\alpha - \boldsymbol{w}_\beta)^\top \boldsymbol{x}_i \frac{1}{T} \sum_{j=1}^T \boldsymbol{x}_j^\top \boldsymbol{x}_i y_j.
    \end{aligned}
\end{equation}

Then we simplify the r.h.s. of inequality (\ref{proof:ineq-3}):
\begin{equation}
    \begin{aligned}
    \label{proof:eq-2}
        &\frac{1}{T} \sum_{i=1}^T\|\boldsymbol{w}_\beta - \boldsymbol{x}_i  y_i\|^2_{\boldsymbol{x}_i \boldsymbol{x}_i^T}
        -
        \|\boldsymbol{w}_\alpha - \boldsymbol{x}_i  y_i\|^2_{\boldsymbol{x}_i \boldsymbol{x}_i^T}\\
        &=\frac{1}{T} \sum_{i=1}^T (\boldsymbol{w}_\beta^\top \boldsymbol{x}_i)^2 - (\boldsymbol{w}_\alpha^\top \boldsymbol{x}_i)^2 + 2(\boldsymbol{w}_\alpha - \boldsymbol{w}_\beta)^\top \boldsymbol{x}_i \boldsymbol{x}_i^\top \boldsymbol{x}_i y_i
    \end{aligned}
\end{equation}

Subtracting Equation (\ref{proof:eq-2}) from Equation (\ref{proof:eq-1}), we get
\begin{equation}
\begin{aligned}
        &\frac{1}{T} \sum_{i=1}^T   2(\boldsymbol{w}_\alpha - \boldsymbol{w}_\beta)^\top \boldsymbol{x}_i \frac{1}{T} \sum_{j=1}^T \boldsymbol{x}_j^\top \boldsymbol{x}_i y_j 
    -
    2(\boldsymbol{w}_\alpha - \boldsymbol{w}_\beta)^\top \boldsymbol{x}_i \boldsymbol{x}_i^\top \boldsymbol{x}_i y_i\\
    &= \frac{1}{T} \sum_{i=1}^T  2(\boldsymbol{w}_\alpha - \boldsymbol{w}_\beta)^\top \boldsymbol{x}_iy_i \frac{1}{T} \sum_{j=1}^T \left(\boldsymbol{x}_j^\top \boldsymbol{x}_i \frac{y_j}{y_i} - \boldsymbol{x}_i^\top \boldsymbol{x}_i \right).
\end{aligned}
\end{equation}

applying the condition 4 in Assumption \ref{ass:theo:algo-selection-context}, we get the final conclusion.

\end{proof}

\section{Limitations}
 1) Most experimental results are based on a GPT-2 model pretrained on a limited set of mathematical functions. It is challenging to assess whether modern large-scale language models like GPT-4 and Claude 3 Opus face similar difficulties in generalizing beyond their pretraining corpus, given the vast range of tasks and content in their pretraining data. Nevertheless, our findings highlight the limitations of ICL in solving challenging tasks for smaller models like Llama-2-7B and Llama-3-8B. 2) The models are trained on ICL data, while real-world LLMs are trained autoregressively. However, the ICL pretraining objective is also next-token prediction, so there is no clear gap between these two pretraining objectives.

\section{Reproducibility}
In the main text and Appendix \ref{app:exp-detail}, we've stated all setups for reproducing our experimental results. For the theoretical part, we've included the assumptions (Assumption \ref{ass:theo:algo-selection-context}) and proofs in Appendix \ref{proof:algo-selection}.

\end{document}